\pdfoutput=1
\documentclass[letterpaper]{article} 
\usepackage{aaai19}  
\usepackage{times}  
\usepackage{helvet}  
\usepackage{courier}  
\usepackage{url}  
\usepackage{graphicx}  
\usepackage{amssymb}
\usepackage{amsmath}
\usepackage{stmaryrd}
\usepackage{mathrsfs}
\usepackage{mathbbol}
\usepackage{placeins}
\usepackage{mathtools,amssymb}
\usepackage{subcaption} 
\usepackage[inline]{enumitem}

\usepackage{tikz}
\usetikzlibrary{arrows,calc}
\tikzset{
	>=stealth',
	help lines/.style={dashed, thick},
	axis/.style={<->},
	important line/.style={thick},
	connection/.style={thick, dotted},
}
\def\defemb#1#2{\expandafter\def\csname #1\endcsname
	{\relax\ifmmode #2\else\hbox{$#2$}\fi}}
\defemb{cA}{{\cal A}}
\defemb{cB}{{\cal B}}
\defemb{cC}{{\cal C}}
\defemb{cc}{{\cal c}}
\defemb{cD}{{\cal D}}
\defemb{cE}{{\cal E}}
\defemb{cF}{{\cal F}}
\defemb{cf}{{\cal f}}
\defemb{cv}{{\cal v}}
\defemb{cx}{{\cal x}}

\defemb{cG}{{\cal G}}
\defemb{GG}{{\cal G}}

\defemb{cI}{{\cal I}}
\defemb{cK}{{\cal K}}
\defemb{cP}{{\cal P}}
\defemb{cL}{{\cal L}}
\defemb{ccl}{{\cal l}}
\defemb{cN}{{\cal N}}
\defemb{cM}{{\cal M}}
\defemb{cR}{{\cal R}}

\defemb{cS}{{\cal S}}
\defemb{SS}{{\cal S}}

\defemb{cT}{{\cal T}}
\defemb{cV}{{\cal V}}
\defemb{cU}{{\cal U}}
\defemb{cO}{{\cal O}}
\defemb{cH}{{\cal H}}
\defemb{cX}{{\cal X}}
\defemb{cY}{{\cal Y}}
\defemb{cZ}{{\cal Z}}

\defemb{sP}{{\sf P}}
\defemb{sG}{{\sf G}}
\defemb{sU}{{\sf U}}
\defemb{sE}{{\sf E}}
\defemb{sI}{{\sf I}}
\defemb{su}{{\sf u}}
\defemb{se}{{\sf e}}
\defemb{ss}{{\sf s}}
\defemb{st}{{\sf t}}
\defemb{sp}{{\sf p}}
\defemb{si}{{\sf i}}
\defemb{sb}{{\sf b}}
\defemb{sw}{{\sf w}}
\defemb{sA}{{\sf A}}
\defemb{sB}{{\sf B}}
\defemb{sa}{{\sf a}}
\defemb{sb}{{\sf b}}
\defemb{scc}{{\sf c}}
\defemb{sd}{{\sf d}}

\defemb{tta}{{\tt a}}
\defemb{ttb}{{\tt b}}
\defemb{ttc}{{\tt c}}
\defemb{ttd}{{\tt d}}
\defemb{tte}{{\tt e}}
\defemb{ttf}{{\tt f}}

\defemb{bara}{{\bar{a}}}
\defemb{barb}{{\bar{b}}}
\defemb{barc}{{\bar{c}}}
\defemb{barf}{{\bar{f}}}
\defemb{barv}{{\bar{v}}}

\newcommand{\leftq}{\llbracket}
\newcommand{\rightq}{\rrbracket}
\newcommand{\defnotate}[1]{{\bf #1}}

\newcommand{\bfbb}{{\sl B\!F\!B\!B}}

\newcommand{\preSet}{{\mathsf{preSet}}}

\newcommand{\sas}{\textsc{sas}${}^{+}$}

\newcommand{\truecost}{h^{\ast}}

\newcommand{\ptask}{\Pi}
\newcommand{\plan}{\pi}

\newcommand{\initstate}{\state_{0}}
\newcommand{\state}{s}
\newcommand{\states}{S}
\newcommand{\vars}{V}
\newcommand{\var}{v}
\newcommand{\val}{d}
\newcommand{\init}{{\initstate}}

\newcommand{\valuefunc}{u}
\newcommand{\budget}{b}
\newcommand{\action}{o}
\newcommand{\actions}{O}
\newcommand{\costfunc}{c}

\newcommand{\domain}{dom}
\newcommand{\eff}{{\mathsf{eff}}}

\newcommand{\pre}{{\mathsf{pre}}}

\newcommand{\actionSeq}{\plan}

\newcommand{\variables}[1]{\cV({#1})}
\newcommand{\applied}[1]{\leftq #1 \rightq}

\newcommand{\brackets}[3]{\left#1#3\right#2}
\newcommand{\tuple}[1]{\brackets{<}{>}{#1}}
\newcommand{\assign}[2]{\tuple{#1/#2}}

\def\reals{{\mathbb R}}

\newcommand{\disland}{L}

\newcommand{\nnreals}{\reals^{0+}}

\usepackage[amsthm,amsmath]{ntheorem}
\newtheorem{theorem}{Theorem}
\newtheorem{lemma}[theorem]{Lemma}
\newtheorem{definition}{Definition}

\newcommand{\uniqNetNegSet}{Y_e}

\newcommand{\netNegPropArg}[2]{y^{#1/#2}}

\newcommand{\negEffectSatisfied}{y^{\var}_{e}}

\newcommand{\unlock}{unlock}
\newcommand{\noOP}{noOP}
\newcommand{\actionName}{\action_{name}}

\begin{document}
%
\title{Automated Tactical Decision Planning Model with Strategic Values Guidance for Local Action-Value-Ambiguity}
\author{Daniel Muller \textsuperscript{1} \and Erez Karpas\\
	The Faculty of Industrial Engineering and Management\\
	Technion - Israel Institute of Technology, Haifa, Israel\\
	\textsuperscript{1} mullerdm@gmail.com, https://mullerd.webgr.technion.ac.il\\
	\textsuperscript{2} karpase@technion.ac.il\\
}
\maketitle
\begin{abstract}

In many real-world planning problems, action's impact differs with a place, time and the context in which the action is applied. The same action with the same effects in a different context or states can cause a different change. In actions with incomplete precondition list, that applicable in several states and circumstances, ambiguity regarding the impact of action is challenging even in small domains. To estimate the real impact of actions, an evaluation of the effect list will not be enough; a relative estimation is more informative and suitable for estimation of action's real impact. Recent work on Over-subscription Planning (OSP) defined the net utility of action as the net change in the state's value caused by the action. The notion of net utility of action allows for a broader perspective on value action impact and use for a more accurate evaluation of achievements of the action, considering inter-state and intra-state dependencies. To achieve value-rational decisions in complex reality often requires strategic, high level, planning with a global perspective and values, while many local tactical decisions require real-time information to estimate the impact of actions. This paper proposes an offline action-value structure analysis to exploit the compactly represented informativeness of net utility of actions to extend the scope of planning to value uncertainty scenarios and to provide a real-time value-rational decision planning tool. The result of the offline pre-processing phase is a compact decision planning model representation for flexible, local reasoning of net utility of actions with (offline) value ambiguity. The obtained flexibility is beneficial for the online planning phase and real-time execution of actions with value ambiguity. Our empirical evaluation shows the effectiveness of this approach in domains with value ambiguity in their action-value-structure. 

\end{abstract}

\section{Introduction}

In many real-world planning and search task there is ``over-subscription" of an action to states or nodes, which raising the concern regarding the increase in search space and brunching factor. In motion and real-time decision making and planning, actions with a degree of freedom in their precondition list, gives the agent flexibility, allowing to apply a compact set of tools to many situations. On the other hand, the flexibility of compact tool set (i.e. actions) that applicable to many situations comes with the complexity to make real-time decisions and to be able to reason about the different impact in different situations. To estimate the real impact of such flexiable actions, evaluation of only the effect list will not be enough. When an action can be applied in many different states, the impact of the same action is relative to the origin state in which the action is applied. Achieving the same outcome from different states can bring damage or wealth, depends on the circumstances and the utility of the origin state . 
In this work we address the ambiguity of action's net utility for actions that applicable in different origin states.

Over-subscription planning (OSP) problem describes many real-world scenarios in which there is ``over-subscription'' of possible achievements to limited resources.~\cite{smith:icaps04,van2004effective,do2004partial,van2004over,nigenda2005planning,benton2006solving,DoBBK:ijcai07,aghighi2014oversubscription,mirkis:domshlak:jair15,muller:Karpas:icaps18}. Scaling up to real-world complexity, with multi-valued arbitrary utility functions over achievements, numerical utility values are challenging even in small domains, due to limited processing capability for inference and utilization of relevant information. To address the complexity of decisions and value-trading inherited in actions, an approach of retaliative estimation of utility of an action is more suitable and informative. Net-benefit planning~\cite{van2004effective,nigenda2005planning,baier2009heuristic,bonet:geffner:aij08,bonet1997robust1,colescoles:icaps11,keyder2009soft} takes a relative estimation approach concerning achievements with an awareness to costs of the action. Recent work defined net utility of actions~\cite{muller:Karpas:icaps18} introduced a new approach to solve OSP problems, taking into account not just the costs but also the net change, i.e. ``what is given away" in order to achieve a change. To define net utility of actions a SAS description tasked assumed, where for each effect of an action a specific precondition is specified. SAS+~\cite{Backstrom:Klein:ci91,Backstrom:Nebel:ci95} representation of tasks allows for a more flexible representation of actions. In SAS+ a precondition list of an action specifies those state variables which must have a certain defined value in order to execute the action and that will also be changed to some other value by the action. This representation allows for an incomplete precondition list which allows for a more compact representation of actions. At the same time, the obtained degree of freedom results with an ambiguity regarding the net change in utility caused by the action. The {\bf selective action split}~\cite{muller:Karpas:icaps18,TR}, manages to determine the net utility of actions perfectly, without unnecessary action split, in most of the domains and tasks in SAS+ by representing compactly the impact of an action in the relative change made by the action with the net utility of an action term. The selective action split label actions with respect to the polarity of the net change caused by the action. However, in several problems, selective action split results in an increase in task size.
This paper proposes a different technique to identify the polarity of actions in SAS+ representation in the context of OSP with arbitrary utility functions over actions. It proposes an offline action-value structure analysis to exploit the compactly represented informativeness of net utility of actions to extend the scope of planning for value uncertainty scenarios, and to provide a real-time value-rational~\cite{muller2018ecosystem} decision planning tool. The result of the offline pre-processing phase is an equivalent OSP task, in which we can determine the net-utility of most of the actions offline. For actions with remained ambiguity regarding the net utility, the offline pre-processing phase provides a compact decision planning model representation for local reasoning during online planning phase which or the real-time execution. Specifically, we introduce an {\em online} approach of action polarity recognition, by deducing the net utility polarity while executing a so called ``unit-effect" actions during search. A set of unit effect actions is generated for each action with a value-structure that containing ambiguity regarding the action's net utility value.  We then show how we can combine our new online technique with the offline selective action split, resulting in better performance than either of them alone.

\section{Background}
We represent OSP model in a language close to \sas\ for classical planning~\cite{Backstrom:Klein:ci91,Backstrom:Nebel:ci95}, an \defnotate{oversubscription planning (OSP)} task is given by a
sextuple {\footnotesize $\ptask = \langle \vars,\initstate,\valuefunc;\actions,\costfunc,\budget\rangle$}, where {\footnotesize $\vars = \{\var_{1},\ldots,\var_{n}\}$} is a finite set of finite-domain {\em state variables\/}, with each complete assignment to {\footnotesize $\vars$} representing a {\em state}, and {\footnotesize $\states = \domain(\var_1)\times\dots\times \domain(\var_n)$} being the {\em state space} of the task; {\footnotesize $\initstate \in \states$} is a designated {\em initial state};
{\footnotesize $\valuefunc$} is an efficiently computable {\em state utility} function {\footnotesize $\valuefunc: \states \rightarrow \reals$}; {\footnotesize $\actions$} is a finite set of {\em actions\/}, with each action {\footnotesize $\action \in \actions$} being represented by a pair {\footnotesize $\langle \pre(\action),\eff(\action)\rangle$} of partial assignments to {\footnotesize $\vars$}, called {\em preconditions} and {\em effects} of {\footnotesize $\action$}, respectively; {\footnotesize $\costfunc: \actions \rightarrow \nnreals$} is an {\em action cost} function; {\footnotesize $\budget \in \nnreals$} is a {\em cost budget} allowed for the task.
An assignment of a variable {\footnotesize $\var$} to a value {\footnotesize $\val$} is denoted by {\footnotesize $\assign{\var}{\val}$} and referred as a {\em fact}. 
For a partial assignment {\footnotesize $p$} to {\footnotesize $\vars$}, let {\footnotesize $\variables{p} \subseteq \vars$} denote the subset of variables instantiated by {\footnotesize $p$}, and, for {\footnotesize $\var\in\variables{p}$}, {\footnotesize $p[\var]$} denote the value provided by {\footnotesize $p$} to the variable {\footnotesize $\var$}.
Action {\footnotesize $\action$} is applicable in a state {\footnotesize $\state$} if {\footnotesize $\state[\var] = \pre(\action)[\var]$} for all {\footnotesize $\var\in\variables{\pre(\action)}$}. Applying {\footnotesize $\action$} changes the value of each {\footnotesize $\var\in\variables{\eff(\action)}$} to {\footnotesize $\eff(\action)[\var]$}, and the resulting state is denoted by {\footnotesize $\state\applied{\action}$}.
A sequence of actions {\footnotesize $\tuple{\action_{1},\dots,\action_{m}}$} denoted by {\footnotesize $\plan$}, called a plan for {\footnotesize $\state$} if it is applicable in {\footnotesize $\state$} and {\footnotesize $\costfunc(\actionSeq) \leq \budget$}. We assume a {\em arbitrary additive utility} function with {\em multi-valued variables}, defined as {\footnotesize $\valuefunc(\state) = \sum_{\assign{\var}{\val}\in\state}{\valuefunc_\var(\val)},$} with {\footnotesize $\valuefunc_\var(\val) \in \reals$} for all variable-value pairs {\footnotesize $\assign{\var}{\val}$}.

{\bf Best-First-Branch-and-Bound ($\bfbb$)} heuristic search for optimal OSP must rely on admissible utility-upper-bounding heuristic function (with budget restrictions) {\footnotesize $h: \states \times \nnreals \rightarrow \nnreals$} to estimate the true utility {\footnotesize $\truecost(\state,\budget)$}. $\bfbb$ also used to solve net-benefit planning problem~\cite{benton2009anytime,benton2012temporal}.

Recent work defined the notion of {\bf net utility value of actions} to increase informativeness and utility of actions in planning domains and solve OSP problems~\cite{muller:Karpas:icaps18,TR}. 
\begin{definition}
	\label{def:actionNetValue}
	For an OSP action {\footnotesize $\action$}, the {\bf net utility} of {\footnotesize $\action$} is {\footnotesize $\valuefunc(\action) = \sum_{\var\in\variables{\eff(\action)}}[{\valuefunc(\eff(\action)[\var])}-{\valuefunc(\pre(\action)[\var])}].$}
\end{definition}

This notion serve to define planning approach interpreting objective as a relative {\em improvement} rather "achieving goals".

\begin{theorem}
	\label{l:timing}
	Given an OSP task {\footnotesize $\ptask$} with a general additive utility function {\footnotesize $\valuefunc$}, for any plan {\footnotesize $\plan$} for {\footnotesize $\ptask$} such that  {\footnotesize $\valuefunc(\state\applied{\plan})>\valuefunc(\init)$}, there is a prefix {\footnotesize $\plan'$} of {\footnotesize $\plan$} such that:
	\begin{enumerate*}
		\item {\footnotesize $\valuefunc(\init\applied{\plan})\le \valuefunc(\init\applied{\plan'})$}, and
		\item for the last action {\footnotesize $\action_{last}$} along {\footnotesize $\plan'$}, we have {\footnotesize $\valuefunc(\action_{last})>0$}.
	\end{enumerate*}
\end{theorem}

In plain word, for each plan {\footnotesize $\plan$}, there is a plan {\footnotesize $\plan'$} that; (i) ends with a positive net utility action, (ii) is at most as costly as {\footnotesize $\plan$}, and (iii) is at least as valuable as {\footnotesize $\plan$}.

\section{Online Detection of the Net Utility Polarity}
We now present a different approach to split actions that differs from the selective action split in the timing of the definition of actions net utility. The selective action split succeeds to define net utility efficiently at the pre-processing stage, but as we mentioned before; sometimes we left with actions that are ambiguous with regard to their net utility signum that could not determined off-line. To handle such actions we can do a normal form encoding and define net utility for each action instance. 
In most of the domains the offline phase detects the net utility for all actions, but real-world scenarios sometimes can be more complex. In such cases as provided in in the example in Figure~\ref{fig:example2}, we can improve the selective action split further by supplying a supplementary mechanism to determine the net utility on-line. 

Online, each action can be applied in few different states, and thus yield different net utilities, depends on the state in which the action is applied. The definition of the action's net utility with a relation to the state in which it is applied changes as follows.

\begin{definition}
	\label{def:actionNetValueNoPre}
	For an OSP action {\footnotesize $\action$} applied in state {\footnotesize $\state$}, the net utility of {\footnotesize $\action$} is 
	{\footnotesize $\valuefunc_{\state}(\action) = \sum_{\var\in\variables{\eff(\action)}}[{\valuefunc_\var(\eff(\action)[\var])}-{\valuefunc_\var(\state[\var])}]$}
\end{definition}

Our approach to that is based on splitting each action into a set of actions, each responsible for achieving a single effect of the original action. For each affected variable with no specified preconditions,  we create a group of actions achieving the same single effect, where each one of them combined with one, different, legal precondition on this variable. If a precondition is specified for the affected variable, then only one action is created which achieves the single effect with one precondition as specified in the original action. 

While applying our approach, the equivalence of the original and the compiled planning tasks must be preserved, which means that for an optimal plan {\footnotesize $\plan$} for the original task {\footnotesize $\ptask$} with solution utility of {\footnotesize $\alpha$}, there exists an optimal plan {\footnotesize $\tilde{\plan}$} for the compiled task {\footnotesize $\tilde{\ptask}$} with the same solution utility of {\footnotesize $\alpha$}, and vice verse. This equivalence is achieved by bounding the execution of the individual actions from each such created action set via dedicated auxiliary control structures that prevent mixing actions from different action sets.

\begin{figure}[t]
	\centering
	\resizebox{\columnwidth}{!}{%
		\begin{tabular}{cc}
			\includegraphics[width=0.6\textwidth]{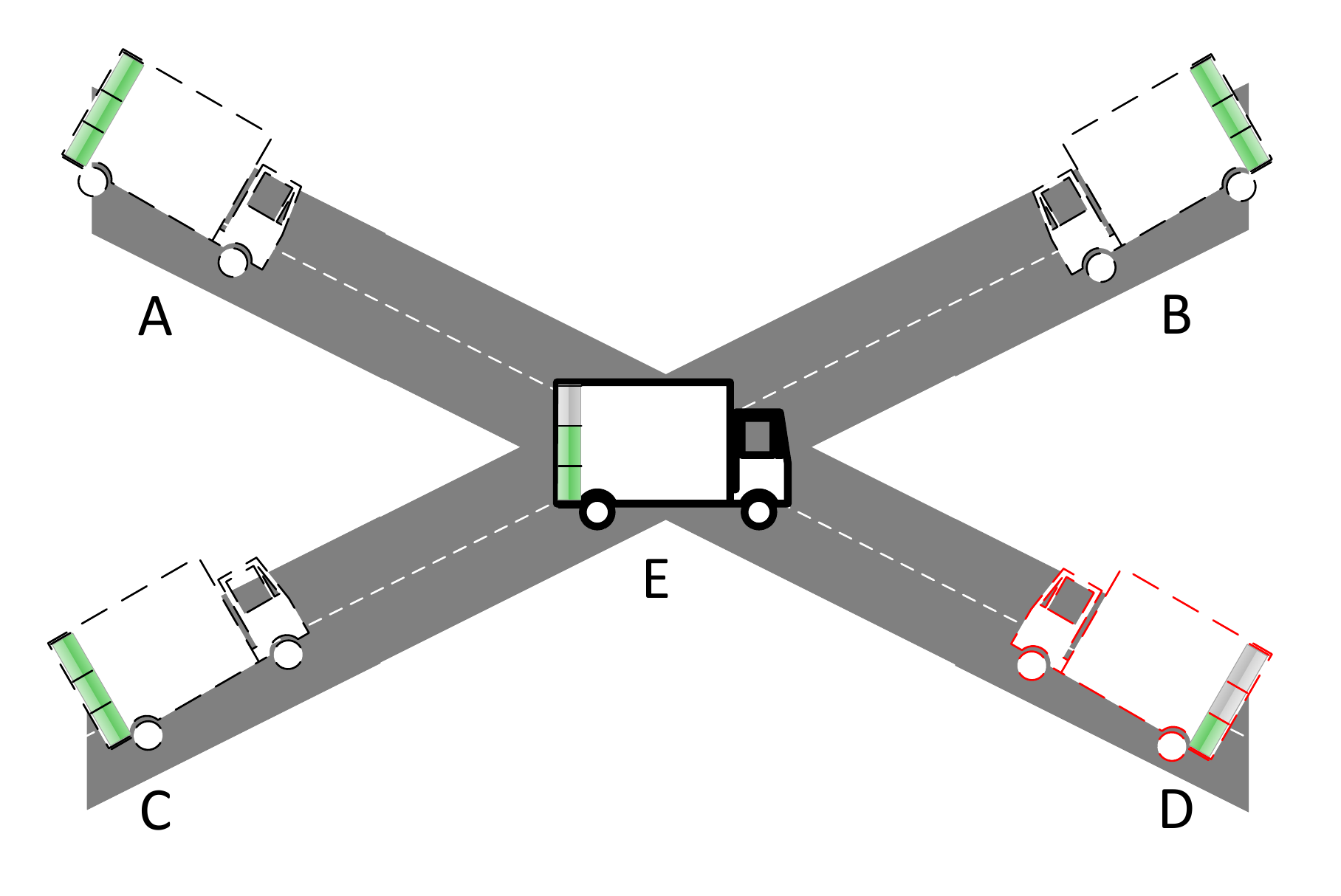}\\
			(a)\\
			\\
			{\footnotesize
				\begin{minipage}{\textwidth}
					\begin{center}
						\resizebox{0.6\columnwidth}{!}{%
						\begin{tabular}{cccccc}
							\multicolumn{1}{c|}{}                  & drive$_{E,2}$                  & drive$_{E,1}$                & drive$_{E,0}$&\\
							\multicolumn{1}{c|}{$\action_i$}       & $i=1$                         & $i=2$                         & $i=3$\\ \hline
							\multicolumn{1}{c|}{$\pre(\action_i)$} & $\{\tuple{f/3}\}$             & $\{\tuple{f/2}\}$             & $\{\tuple{f/1}\}$ &\\
							\multicolumn{1}{c|}{$\eff(\action_i)$} & $\{\tuple{t/E},\tuple{f/2}\}$             & $\{\tuple{t/E},\tuple{f/1}\}$             & $\{\tuple{t/E},\tuple{f/0}\}$&\\
						\end{tabular}
					}
					\end{center}
				\end{minipage}
			} 
			\\
			\ \\
			(b)\\
			\\
			{\footnotesize
				\begin{minipage}{\textwidth}
					\begin{center}
						\resizebox{\columnwidth}{!}{%
						\begin{tabular}{ccccccc}
							\multicolumn{1}{c|}{drive$_{E,2}$ set}       & drive$_{AE,2}$                  & drive$_{BE,2}$                  & drive$_{CE,2}$\\ 
							\multicolumn{1}{c|}{$\action_i$}       & $i=1$                         & $i=2$                         & $i=3$\\ \hline
							\multicolumn{1}{c|}{$\pre(\action_i)$} & $\{\tuple{t/A}, \tuple{\unlock/drive_{E,2}}\}$             & $\{\tuple{t/B}, \tuple{\unlock/drive_{E,2}}\}$             & $\{\tuple{t/C}, \tuple{\unlock/drive_{E,2}}\}$        	\\
							\multicolumn{1}{c|}{$\eff(\action_i)$} & $\{\tuple{t/E}\}$             & $\{\tuple{t/E}\}$             & $\{\tuple{t/E}\}$\\
						\end{tabular}
					}
					\end{center}
				\end{minipage}
			} 
			\\\\
			{\footnotesize
				\begin{minipage}{\textwidth}
					\begin{center}
						\resizebox{\columnwidth}{!}{%
						\begin{tabular}{ccccccc}
							\multicolumn{1}{c|}{drive$_{E,2}$ set} & reduceFuel$_{3 \to 2}$      &  drive$^{E,2}_{unlock}$     		& drive$^{E,2}_{lock}$ \\ 
							\multicolumn{1}{c|}{$\action_i$} & $i=4$     & $i=5$  & $i=6$\\ \hline
							\multicolumn{1}{c|}{$\pre(\action_i)$} & $\{\tuple{f/3}, \tuple{\unlock/drive_{E,2}}\}$             & $\{\tuple{f/3}\}$             &$\{\tuple{\unlock/drive_{E,2}}, \tuple{t/E}, \tuple{f/3} \}$             \\
							\multicolumn{1}{c|}{$\eff(\action_i)$} & $\{\tuple{f/2}\}$ & $\{\tuple{\unlock/drive_{E,2}}\}$  & $\{\tuple{\unlock/\noOP}\}$\\
						\end{tabular}
					}
					\end{center}
				\end{minipage}
			} 
			\\
			\\
			(c)\\			
		\end{tabular}}
	\caption{\label{fig:example2} An example of an OSP task, with (a) illustrating the story, (b) listing the original actions, and (c) details a set of actions compiled from the original action drive{\footnotesize $_{E,2}$}. Each action in (c) achieves a single effect from the effect list of action drive{\footnotesize $_{E,2}$}. actions {\footnotesize $\action_{5}$} and {\footnotesize $\action_{6}$} are auxiliary control predicates. {\footnotesize $\action_{5}$} enables the new single effect actions in drive{\footnotesize $_{E,2}$} set, while preventing from all other sets to be applied. {\footnotesize $\action_{6}$} approves that all effect are achieved, locking current action set and enables starting new action set.  Note, the truck in point D can reach point E, but since it does not meet the fuel precondition for the original action drive{\footnotesize $_{E,2}$} it is not included in the equivalent single action set of drive{\footnotesize $_{E,2}$} in (c).}
\end{figure}

An example of a simple OSP task in Figure~\ref{fig:example2} is used to illustrate our approach of splitting the original action into set of single effect actions, and the use of the auxiliary control structure action set. In this example, we have a truck that can move between locations A, B, C, D and E. The truck has three levels of fuel and each step reduces the fuel by one level. As we can see in this example, a truck which is initially at one of the locations A, B, C and D, with fuel level of three can end up in the same location (E) and with the same fuel level (two). In this case we must have one specific precondition on the fuel level, but for the location there are few possible values as a precondition. We create an action for each one of those possible preconditions.

This OSP task {\footnotesize $\ptask$} described here using two state variables {\footnotesize $\vars = \{t, y, \unlock \}$} with {\footnotesize $\domain(t) = \{A, B, C, D, E \}$, $\domain(f) = \{0, 1, 2, 3\}$} and {\footnotesize $\domain(\unlock) = \{drive_{E,2}, drive_{E,1},drive_{E,0}, \noOP\}$}, where {\footnotesize $t$} stands for the possible location of truck and {\footnotesize $f$} for the fuel level, {\footnotesize $\unlock$} is a control predicate which enables/disables single effect actions sets. For example when drive{\footnotesize $_{E,2}$} set is enabled, all other sets will be disabled. This way we avoid mixing single effect actions from different sets and preserve equivalence. A group of available actions for this example is detailed in  Figure~\ref{fig:example2}(b), and the created, single effect action set for one of those actions (drive{\footnotesize $_{E,2}$}) is detailed in Figure~\ref{fig:example2}(c) along with their auxiliary control predicates. 
In the state model induced by this OSP task, 
we have {\footnotesize $\states = \domain(t) \times \domain(f) \times \domain(\unlock)$}, and the predicate values:
{\footnotesize $\valuefunc(\unlock) = 0$,\\ $\valuefunc(t) = \begin{cases}
2, \;\;\;\; t = E\\
0, \;\;\;\; \text{otherwise}
\end{cases}$, $\valuefunc(f) = \begin{cases}
1, \;\;\;\; f = 1 \\
2, \;\;\;\; f = 2\\
3, \;\;\;\; f = 3\\
0, \;\;\;\; \text{otherwise}
\end{cases}$\\
}

Now let us look at the most basic example of single action plan, presenting our approach. Consider the OSP task {\footnotesize $\ptask$} and its compiled task {\footnotesize $\tilde{\ptask}$}, described in Figure~\ref{fig:example2}. The initial state {\footnotesize $\initstate=\text{A3}$}, i.e. truck is in location A with fuel level of 3. A valid single action plan from this initial state is:
{\footnotesize $\plan = \tuple{drive_{AE,2}}$}
using Definition~\ref{def:actionNetValueNoPre} the net utility will be:
{\footnotesize $\valuefunc(\plan) = \valuefunc(drive_{AE,2}) = [\tuple{{t/E}} - \tuple{{t/A}}] + [\tuple{{f/2}} - \tuple{{f/3}}]  = 1$}
one of the sequences which are equivalent plans in the compiled task {\footnotesize $\tilde{\ptask}$} is: 
{\footnotesize $\tilde{\plan_{1}} = \tuple{drive^{E,2}_{unlock},\; Fuel_{3 \to 2},\; drive_{AE,2},\; drive^{E,2}_{lock}}$}
with a net utility of:
{\footnotesize $\valuefunc(\tilde{\plan}) = \valuefunc(drive^{E,2}_{unlock}) + \valuefunc(Fuel_{3 \to 2}) +\valuefunc(drive_{AE,2}) +\valuefunc(drive^{E,2}_{unlock})= 0 + [\tuple{{f/2}} - \tuple{{f/3}}] + [\tuple{{t/E}} - \tuple{{t/A}}] + 0 =0 -1+ 2 + 0 = 1$}
But what if we switch the order of the actions {\footnotesize $Fuel_{3 \to 2}$} and  {\footnotesize $drive_{AE,2}$} in this sequence? Let us look on the following sequence;
{\footnotesize $\tilde{\plan_{2}} = \tuple{drive^{E,2}_{unlock},\; drive_{AE,2},\; Fuel_{3 \to 2},\; drive^{E,2}_{lock}}$}
In this case, at the end of the applied sequence we get the same net utility accumulated in a different manner, i.e. {\footnotesize $(2 -1= 1)$} instead of {\footnotesize $(-1 +2 = 1)$}. Apparently both {\footnotesize $\tilde{\plan_{1}}$} and {\footnotesize $\tilde{\plan_{2}}$} are equivalent to {\footnotesize $\plan$}, but while examining the action sequence in {\footnotesize $\tilde{\plan_{2}}$} one can see that after applying the action {\footnotesize $drive_{AE,2}$} the utility exceeds the optimal plan {\footnotesize $\plan$} utility reached in the original task {\footnotesize $\ptask$}, while reaching invalid utility in the original task, as illustrated in Figure~\ref{fig:ex.over.opt}. Thus {\footnotesize $\tilde{\plan_{2}}$} is not equivalent to plan {\footnotesize $\plan$}.

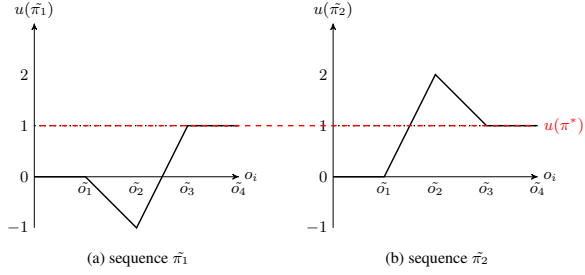
\begin{figure}[t]
		\centering
        {\footnotesize
		\resizebox{0.95\columnwidth}{!}{%
		\begin{tabular}{cc}
			\begin{subfigure}[b]{0.3\textwidth}
				\begin{tikzpicture}[remember picture]
				\coordinate (op1) at (1,0);
				\coordinate (op2) at (2,0);
				\coordinate (op3) at (3,0);
				\coordinate (op4) at (4,0);			
				\coordinate (valmin1) at (0,-1);
				\coordinate (val0) at (0,0);
				\coordinate (val1) at (0,1);
				\coordinate (val2) at (0,2);				
				\coordinate (label) at (2,-1);				
				\draw (op1) node [below] {$\tilde{\action_{1}}$};
				\draw (op2) node [below] {$\tilde{\action_{2}}$};
				\draw (op3) node [below] {$\tilde{\action_{3}}$};
				\draw (op4) node [below] {$\tilde{\action_{4}}$};		
				\draw (valmin1) node [left] {$-1$};
				\draw (val0) node [left] {$0$};
				\draw (val1) node [left] {$1$};
				\draw (val2) node [left] {$2$};	
				\draw [->] (0,0) -- (4,0) node [right]  {$\action_{i}$};
				\draw [->] (0,-1) -- (0,3) node [above] {$\valuefunc(\tilde{\plan_{1}})$};
				\draw[thick] (0,0) -- (1,0) -- (2,-1) -- (3,1) -- (4,1);
				\draw[connection] let \p1=(0,1), \p2=(op4) in (0,1)  -- (4,1)node[right]{};			
				\node (red-1) at (0,1) {}; 		
				\end{tikzpicture}
				\caption{sequence $\tilde{\plan_{1}}$}
				\label{fig:A}
			\end{subfigure}
			\;\;\;\;
			\begin{subfigure}[b]{0.30\textwidth}
				\centering
				\begin{tikzpicture}[remember picture]
				\coordinate (op1) at (1,0);
				\coordinate (op2) at (2,0);
				\coordinate (op3) at (3,0);
				\coordinate (op4) at (4,0);			
				\coordinate (valmin1) at (0,-1);
				\coordinate (val0) at (0,0);
				\coordinate (val1) at (0,1);
				\coordinate (val2) at (0,2);				
				\draw (op1) node [below] {$\tilde{\action_{1}}$};
				\draw (op2) node [below] {$\tilde{\action_{2}}$};
				\draw (op3) node [below] {$\tilde{\action_{3}}$};
				\draw (op4) node [below] {$\tilde{\action_{4}}$};			
				\draw (valmin1) node [left] {$-1$};
				\draw (val0) node [left] {$0$};
				\draw (val1) node [left] {$1$};
				\draw (val2) node [left] {$2$};			
				\coordinate (x2) at (4,0);
				\draw [->] (0,0) -- (4,0) node [right]  {$\action_{i}$};
				\draw [->] (0,-1) -- (0,3) node [above] {$\valuefunc(\tilde{\plan_{2}})$};
				\draw[thick] (0,0) -- (1,0) -- (2,2) -- (3,1) -- (4,1);
				\draw[connection] let \p1=(0,1), \p2=(x2) in (0,1)  -- (\x2,1)node[right, red]
				{$\valuefunc(\plan^\ast)$};			
				\node (red-2) at (4,1) {}; 		
				\end{tikzpicture}
				\caption{sequence $\tilde{\plan_{2}}$}
				\label{fig:B}
			\end{subfigure}
			\begin{tikzpicture}[remember picture,overlay]
			\draw[dashed,red, thick] (red-1) -- (red-2);
			\end{tikzpicture}
		\end{tabular} 
	}
	}
	\caption{\label{fig:ex.over.opt} Graphs (\protect\subref{fig:A}), (\protect\subref{fig:B}) present the cumulative actions' utility of sequences {\footnotesize $\tilde{\plan_{1}}$} and {\footnotesize $\tilde{\plan_{2}}$}, respectively. In sequence {\footnotesize $\tilde{\plan_{1}}$} net negative action applied prior to net positive, and in sequence {\footnotesize $\tilde{\plan_{2}}$} vice verse. While both {\footnotesize $\tilde{\plan_{1}}$} and {\footnotesize $\tilde{\plan_{2}}$} reach same final utility, {\footnotesize $\tilde{\plan_{2}}$} exceeds the optimal utility value in the original task, which violates the equivalence between {\footnotesize $\ptask$} and {\footnotesize $\tilde{\ptask}$}.}
\end{figure}

To preserve the equivalence with the original task, we have to prevent from exceeding the original optimal utility value during the execution of the compiled 
single-effect action set, as illustrated in Figure~\ref{fig:ex.over.opt}. To promise equivalence, we extend the control structure in the compiled task (presented in Figure~\ref{fig:example2}) with a set of auxiliary control predicates and actions, dedicated to verify that the net negative utility actions applied before the net positive actions, as in {\footnotesize $\tilde{\plan_{1}}$}.

\begin{figure}[t]
	\centering
	\resizebox{\columnwidth}{!}{%
		\begin{tabular}{cc}
			{
					\begin{tabular}{ccccccc}
						\multicolumn{1}{c|}{drive$_{E,2}$ set} & reduceFuel$_{3 \to 2}^{\color{red} verify}$      \\ 
						\multicolumn{1}{c|}{$\action_i$}       & $i=4$    \\ \hline
						\multicolumn{1}{c|}{$\pre(\action_i)$} & $\{\tuple{f/3}, \tuple{\unlock/drive_{E,2}}\}$                    	\\
						\multicolumn{1}{c|}{$\eff(\action_i)$} & $\{\tuple{f/2}, {\color{red} \tuple{\netNegPropArg{f}{2}/1}}\}$ \\
					\end{tabular}
			} 
			\\
			(a)\\
			\\
			{
					\begin{tabular}{ccccccc}
						\multicolumn{1}{c|}{drive$_{E,2}$ set}       & {\color{red} drive$_{AE,2}^{verifyNO}$}                  & {\color{red} drive$_{BE,2}^{verifyNO}$}                  & {\color{red} drive$_{CE,2}^{verifyNO}$}\\ 
						\multicolumn{1}{c|}{ $\action_i$}       & {\color{red} $i=7$}                         & {\color{red} $i=8$}                         & {\color{red} $i=9$}\\ \hline
						\multicolumn{1}{c|}{$\pre(\action_i)$} & {\color{red} $\{\tuple{t/A}, \tuple{\unlock/drive_{E,2}}\}$}             &{\color{red}  $\{\tuple{t/B}, \tuple{\unlock/drive_{E,2}}\}$}             & {\color{red} $\{\tuple{t/C}, \tuple{\unlock/drive_{E,2}}\}$}        	\\
						\multicolumn{1}{c|}{$\eff(\action_i)$} & {\color{red} $\{ \tuple{\netNegPropArg{t}{E}/1}\}$}             & {\color{red} $\{\tuple{\netNegPropArg{t}{E}/1}\}$}             & {\color{red} $\{\tuple{\netNegPropArg{t}{E}/1}\}$}\\
					\end{tabular}
			} 
			\\
			(b)\\
			\\		
			{
					\begin{tabular}{ccccccc}
						\multicolumn{1}{c|}{drive$_{E,2}$ set}       & drive$_{AE,2}^{\color{red} +}$                  & drive$_{BE,2}^{\color{red} +}$\\ 
						\multicolumn{1}{c|}{$\action_i$}       & $i=1$                         & $i=2$\\ \hline
						\multicolumn{1}{c|}{$\pre(\action_i)$} & $\{\tuple{t/A}, \tuple{\unlock/drive_{E,2}}, {\color{red} \tuple{\netNegPropArg{t}{E}/1}}, {\color{red} \tuple{\netNegPropArg{f}{2}/1}}\}$             & $\{\tuple{t/B}, \tuple{\unlock/drive_{E,2}}, {\color{red} \tuple{\netNegPropArg{t}{E}/1}}, {\color{red} \tuple{\netNegPropArg{f}{2}/1}}\}$\\
						\multicolumn{1}{c|}{$\eff(\action_i)$} & $\{\tuple{t/E}\}$             & $\{\tuple{t/E}\}$\\
					\end{tabular}
			} 
			\\\\			
			{
					\begin{tabular}{ccccccc}
						\multicolumn{1}{c|}{drive$_{E,2}$ set} & drive$_{CE,2}^{\color{red} +}$ \\ 
						\multicolumn{1}{c|}{$\action_i$}       & $i=3$ \\ \hline
						\multicolumn{1}{c|}{$\pre(\action_i)$} & $\{\tuple{t/C}, \tuple{\unlock/drive_{E,2}}, {\color{red} \tuple{\netNegPropArg{t}{E}/1}}, {\color{red} \tuple{\netNegPropArg{f}{2}/1}}\}$                	\\
						\multicolumn{1}{c|}{$\eff(\action_i)$} & $\{\tuple{t/E}\}$ \\
					\end{tabular}
			} 
			\\
			(c)\\
			\\		
			{
					\begin{tabular}{ccccccc}
						\multicolumn{1}{c|}{drive$_{E,2}$ set} &  drive$^{E,2}_{unlock}$     		& drive$^{E,2}_{lock}$ \\ 
						\multicolumn{1}{c|}{$\action_i$} & $i=5$  & $i=6$\\ \hline
						\multicolumn{1}{c|}{$\pre(\action_i)$} & $\{\tuple{f/3}, \tuple{\unlock/\noOP}\}$             &$\{\tuple{\unlock/drive_{E,2}}, \tuple{t/E}, \tuple{f/3}, {\color{red} \tuple{\netNegPropArg{t}{E}/1}}, {\color{red} \tuple{\netNegPropArg{f}{2}/1}}\}$             \\
						\multicolumn{1}{c|}{$\eff(\action_i)$} & $\{\tuple{\unlock/drive_{E,2}}\}$  & $\{\tuple{\unlock/\noOP}, {\color{red} \tuple{\netNegPropArg{t}{E}/0}}, {\color{red} \tuple{\netNegPropArg{f}{2}/0}}\}$\\
					\end{tabular}
			} 		
			\\
			(d)\\			
		\end{tabular}
	}
	\caption{\label{fig:example_with_verify} A set of actions compiled for the action drive{\footnotesize $_{E,2}$}, extended with an auxiliary control structure for preserving equivalence with the original task {\footnotesize $\ptask$}.}
\end{figure}

The set of actions detailed in Figure~\ref{fig:example_with_verify} is used to illustrate the extended auxiliary control structure which is used to prevent exceeding optimal plan {\footnotesize $\plan$} utility reached in the original task {\footnotesize $\ptask$}. We add a new group of predicates {\footnotesize $\uniqNetNegSet = \{\negEffectSatisfied \mid \var\in \vars, \{e\} \in \domain(\var) \}$} with {\footnotesize$ \domain(\negEffectSatisfied) = \{0,1\}$}, where for each action {\footnotesize $\action$} in the original task {\footnotesize $\ptask$}, we add {\footnotesize $\{\negEffectSatisfied \in \uniqNetNegSet \mid \tuple{\var/e} \in \eff(\action)\}$} to {\footnotesize $\action$}'s equivalent set of single effect actions in task {\footnotesize $\tilde{\ptask}$}. The semantics of {\footnotesize $\tuple{\negEffectSatisfied/1} \in \state$} is that either effect {\footnotesize $e$} has been collected by applying net negative utility action or there is no valid net negative action which achieves effect {\footnotesize $e$}. When the net positive actions in the set defined as; 
{\footnotesize
$\pre(\action^+_{p\to e}) = \left \{ \tuple{\var/p}, \tuple{\unlock/\actionName} \right \}
\bigcup_{\negEffectSatisfied \in \uniqNetNegSet \mid \tuple{\var/e} \in \eff(\action)}\{\tuple{\negEffectSatisfied/1}\},\\
\eff(\action^+_{p\to e}) = \left \{
\tuple{\var/e}\right \}, 
\tilde{\costfunc}(\action^+_{p\to e}) = 0.$
}

This way, net positive utility action can be applied and positive utility value obtained only if all net negative actions has been applied and all negative utility carrying facts that hold in the current state and set has been collected.

To enforce this semantics of {\footnotesize $\negEffectSatisfied$}, the action set {\footnotesize $\actions_{set,o}$} contains a pair of action sets, {\footnotesize $\action^{verify}_{p\to e}$} which verifies that negative utility carrying value should be collected and is collected via net negative utility action, while {\footnotesize $\action^{verifyNo}_{p\to e}$} verifies that the effect can't be achieved by net negative utility action, where

{\centering\footnotesize 
$\bigg\langle\begin{aligned}
\pre(\action^{verify}_{p\to e}) &= \left \{ \tuple{\var/p}, \tuple{\unlock/\actionName} \right \},\\
\eff(\action^{verify}_{p\to e}) &= \left \{ \tuple{\var/e}, \tuple{\negEffectSatisfied/1}\right \}, \tilde{\costfunc}(\action^{verify}_{p\to e}) = 0\\
\end{aligned}\bigg\rangle$,\\
$\bigg\langle\begin{aligned}
\pre(\action^{verifyNo}_{p\to e}) &= \left \{ \tuple{\var/p}, \tuple{\unlock/\actionName} \right \},\\
\eff(\action^{verifyNo}_{p\to e}) &= \left \{ \tuple{\negEffectSatisfied/1}\right \}, \tilde{\costfunc}(\action^{verifyNo}_{p\to e}) = 0
\end{aligned}\bigg\rangle$.\\
}

Figure~\ref{fig:example_with_verify} depicts an example of the set of actions obtained by the compilation for the original action drive{\footnotesize $_{E,2}$} with extended, dedicated auxiliary control structure for preserving equivalence with the original task {\footnotesize $\ptask$}. The extensions beyond the action set in Figure~\ref{fig:example2} are colored red, where {\bf (a)} is a net negative action that achieves {\footnotesize $\tuple{\var/e} \in \eff(\action)$} and verifies that it has been achieved via net negative utility action by collecting predicate {\footnotesize $\negEffectSatisfied$}; {\bf (b)} are actions that verify that {\footnotesize $\tuple{\var/e} \in \eff(\action)$} cannot be achieved via net negative utility action by collecting predicate {\footnotesize $\negEffectSatisfied$}; {\bf (c)} are net positive utility actions that achieve {\footnotesize $\tuple{\var/e} \in \eff(\action)$} and can be applied only after either all the net negative utility actions have been applied, or it has been verified that there is no net negative utility action which achieves {\footnotesize $\tuple{\var/e} \in \eff(\action)$} by preconditioning with {\footnotesize $\bigcup_{\negEffectSatisfied \in \uniqNetNegSet \mid \tuple{\var/e} \in \eff(\action)} \{\tuple{\negEffectSatisfied/1}\}$}; {\bf (d)} are control actions that enable/disable drive{\footnotesize $_{E,2}$} set, where drive{\footnotesize $^{E,2}_{lock}$} resets the auxiliary predicates {\footnotesize $\negEffectSatisfied$} when finishing the sequence. 
	

Using the extended auxiliary control structure for preserving equivalence of plan {\footnotesize $\tilde{\plan}$} for the compiled task {\footnotesize $\tilde{\ptask}$} with the original task {\footnotesize $\ptask$}, {\footnotesize $\tilde{\plan}$} will be; {\scriptsize $\tilde{\plan_{1}} = \langle drive^{E,2}_{unlock}, Fuel_{3 \to 2}^{verify}, drive_{AE,2}^{verifyNO},drive_{AE,2}^{+}, drive^{E,2}_{lock}\rangle$}.
The net positive utility action $drive_{AE,2}^{+}$ can be applied only after all $\tuple{\var/e} \in \eff(\action)$ been verified that either they been collected by applying net negative utility action or there is no valid net negative action which achieves them. Definition~\ref{def:unconstrainedPreCompOnFly} puts together the above observations to work.
\begin{definition}
{\footnotesize
	\label{def:unconstrainedPreCompOnFly}
		Let $\ptask  =\langle \vars,\initstate,\valuefunc;\actions,\costfunc,\budget\rangle$ be an OSP task, and {\bf unit-effect-compilation} of $\ptask$ is an OSP task $\tilde{\ptask} = \langle \tilde{\vars},\tilde{\initstate},\tilde{\valuefunc};\tilde{\actions},\tilde{\costfunc},\tilde{\budget} \rangle$ where $\tilde{\vars} = \vars \cup \{\unlock\} \cup \uniqNetNegSet$ with $\uniqNetNegSet = \{\negEffectSatisfied \mid \var\in \vars, e \in \domain(\var) \}$, $\domain(\negEffectSatisfied) = \{0,1\}$, $\domain(\unlock) = \{\actionName \mid \action \in \actions\} \cup \{\noOP\}$, $\tilde{\init} = \init \cup \{ \tuple{\unlock/\noOP}\} \cup \{\tuple{y/0} \mid y \in \uniqNetNegSet\}$, $\tilde{\valuefunc} = \valuefunc$, $\tilde{\budget} = \budget$, $\tilde{\actions} =\actions_{AllSets} \cup \actions_{unlock} \cup \actions_{lock}$ where
       }
{\scriptsize 			
		$\actions_{AllSets} =   \bigcup_{\action \in \actions} \actions_{set,o}, \;\; with \;\; \actions_{set,o}  = \\
             \begin{aligned}
                &\;\ \{ \action^{verify}_{p\to e} \mid \action \in \actions, p\in \preSet_\var(\action) , \tuple{\var/e} \in \eff(\action),\\
                &\;\;\;\;\;\;\;\; \valuefunc_{\var}(e) - \valuefunc_{\var}(p) \leq 0 \} \;\ \cup\\
                &\;\{ \action^{verifyNo}_{p\to e} \mid \action \in \actions, p\in \preSet_\var(\action) , \tuple{\var/e} \in \eff(\action),\\
                &\;\;\;\;\;\;\;\; \valuefunc_{\var}(e) - \valuefunc_{\var}(p) >0  \} \;\ \cup\\
                &\;\{\action^+_{p\to e} \mid \action \in \actions,p\in \preSet_\var(\action)  , \tuple{\var/e} \in \eff(\action), \\
                &\;\;\;\;\;\;\;\; \valuefunc_{\var}(e) - \valuefunc_{\var}(p) >0 \} \\
                & where:\\
                &\;\ \preSet_\var(\action) = 
                \begin{cases}
                    \pre(\action)[\var], & \var\in\variables{\pre(\action)}\\
                    \domain(\var),& \var\notin\variables{\pre(\action)}\\
                \end{cases}\\
              \end{aligned}$
              \[
                    \begin{split}
                    &\pre(\action^{verify}_{p\to e}) = \left \{ \tuple{\var/p}, \tuple{\unlock/\actionName} \right \}\\
                    &\eff(\action^{verify}_{p\to e}) = \left \{ \tuple{\var/e}, \tuple{\negEffectSatisfied/1}\right \},\tilde{\costfunc}(\action^{verify}_{p\to e}) = 0\\
                    &\pre(\action^{verifyNo}_{p\to e}) = \left \{ \tuple{\var/p}, \tuple{\unlock/\actionName} \right \}\\
                    &\eff(\action^{verifyNo}_{p\to e}) = \left \{ \tuple{\negEffectSatisfied/1}\right \}, \tilde{\costfunc}(\action^{verifyNo}_{p\to e}) = 0\\
                    &\pre(\action^+_{p\to e}) = \left \{ \tuple{\var/p}, \tuple{\unlock/\actionName} \right \}  \bigcup_{\negEffectSatisfied \in \uniqNetNegSet \mid \tuple{\var/e} \in \eff(\action)} \{\tuple{\negEffectSatisfied/1}\}\\
                    &\eff(\action^+_{p\to e}) = \left \{
                    \tuple{\var/e}\right \}, \tilde{\costfunc}(\action^+_{p\to e}) = 0\\
                    \end{split}
			\] 
		$\actions_{unlock} =  \left\{\action_{unlock} \mid \action \in \actions \right\}$ with $\tilde{\costfunc}(\action_{unlock}) = \costfunc(\action)$ and 
		\[
		\begin{split}
		&\pre(\action_{unlock}) =  \pre(\action) \cup \{\tuple{\unlock/\noOP} \}\\
		&\eff(\action_{unlock}) =  \{\tuple{\unlock/\actionName}\} 
		\end{split}
		\]
		\item  $\actions_{lock} =  \left\{\action_{lock} \mid \action \in \actions \right\}$ with $\tilde{\costfunc}(\action_{lock}) = 0$ and, 
		\[
		\begin{split}
		&\pre(\action_{lock}) =  \eff(\action) \cup \{\tuple{\unlock/\actionName}\} \bigcup_{\negEffectSatisfied \in \uniqNetNegSet \mid \tuple{\var/e} \in \eff(\action)} \{\tuple{\negEffectSatisfied/1}\}\\
		&\eff(\action_{lock}) =  \{\tuple{\unlock/\noOP} \} \bigcup_{\negEffectSatisfied \in \uniqNetNegSet \mid \tuple{\var/e} \in \eff(\action)} \{\tuple{\negEffectSatisfied/0}\}\\
		\end{split}
		\]            
}			
\end{definition}

In plain words,  {\footnotesize $\tilde{\ptask}$} extends the structure of {\footnotesize $\ptask$} by
\begin{itemize}
	\item Converting {\footnotesize $\action \in \actions$} into a set of single-effect actions {\footnotesize $\actions_{set,o}$}.
	\item Adding control actions {\footnotesize $\action_{unlock}$}. When all the precondition for the original action {\footnotesize $\action \in \actions$} hold, this action permits execution of actions {\footnotesize $\tilde{\action} \in \actions_{set,o}$}, while locking the ability to perform any other action {\footnotesize $\tilde{\action'} \in \actions_{set,o'}$} such that {\footnotesize $\action' \neq \action$}.   
	\item Adding control actions {\footnotesize $\action_{lock}$}. When all the effects of action {\footnotesize $\action$} are achieved, this action locks the current {\footnotesize $\actions_{set,o}$} and allows starting new single effect actions set {\footnotesize $\actions_{set,o}$}, by nullifying control fact {\footnotesize $\unlock$}.
\end{itemize}

\begin{definition}
	\label{def:sequence}
	For any action {\footnotesize $\action \in \actions$} of an OSP task {\footnotesize $\ptask$}, then actions sequence {\footnotesize $\tilde{\plan}_{o}$} is defined as 
	{\footnotesize $\tilde{\plan}_{o} = \tuple{\action_{unlock},\tilde{\action_{1}},\dots, \tilde{\action_{k}}, \action_{lock}},$} where {\footnotesize $\tilde{\action_{1}},\dots, \tilde{\action_{k}} \in \actions_{set,o}$} are single-effect actions.
\end{definition}

\begin{lemma}
	\label{l:separated}
	Given a valid plan {\footnotesize $\tilde{\plan}$} for an OSP task {\footnotesize $\tilde{\ptask}$}, any action sequence {\footnotesize $\tilde{\plan}_{o}$} in plan {\footnotesize $\tilde{\plan}$} is an {\bf atomic action sequence}, where {\footnotesize $\nexists \tilde{{\action'}} \mid \tilde{{\action}}'  \in \tilde{\plan}_{o}, \tilde{{\action}}' \notin \actions_{set,{\action}}\cup \{\action_{unlock},\action_{lock}\}, \action \ne \action'.$}
\end{lemma}

\begin{proof}
	The proof is by case analysis. 
	Given an OSP task {\footnotesize $\ptask$}, let {\footnotesize $\tilde{\ptask}$} be the {unit-effect-compilation} of {\footnotesize $\ptask$}, where
	{\footnotesize $\actions_{set,{\action_{1}}}\dots\actions_{set,{\action_{n}}}$} are single-effect action sets created from the original action {\footnotesize $\action_{i} \in \actions$} in {\footnotesize $\ptask$}, respectively. Let {\footnotesize $\tilde{\plan} = \tuple{\tilde{\plan}_{\action_{1}}, \tilde{\plan}_{\action_{2}}, \dots ,\tilde{\plan}_{\action_{n}}}$} a valid plan for an OSP task {\footnotesize $\tilde{\ptask}$}, we show that,  through fact and action mutex relationships, the auxiliary control structure in task {\footnotesize $\tilde{\ptask}$} ensures the  atomic execution of action sequences in each step of plan execution. In order to complete the proof we have to go through the following cases: {\bf Case 1} in the initial state {\footnotesize $\init$}, {\bf Case 2} at sequence initiation, {\bf Case 3} during sequence execution, and, {\bf Case 4} at sequence termination.
	By the construction of {\footnotesize $\tilde{\ptask}$} the following holds;
	\begin{enumerate}
		\item For any {\footnotesize $\tilde{\action}\in\tilde{\actions}$} the variable {\footnotesize $\unlock\in\variables{\pre(\tilde{\action})}$} with a domain {\footnotesize $\domain(\unlock) = \{\actionName \mid \action \in \actions\} \cup \{\noOP\}$}, the facts {\footnotesize $\{\tuple{\unlock/\actionName}\}\cup \{\noOP\} \mid \action \in \actions$} are pairwise mutex, hence any two actions {\footnotesize $\tilde{\action}, \tilde{\action}' \in \tilde{\actions}$} such that {\footnotesize $\assign{\unlock}{p} \in \pre(\tilde{\action}), \assign{\unlock}{{p'}} \in \pre(\tilde{\action}')$} where {\footnotesize $p \neq p'$} mutex as well. 
		\item For any action set {\footnotesize $\actions_{set,{\action}}\in\tilde{\actions}$}, and for any action {\footnotesize $\tilde{\action}\in\actions_{set,o}$} it holds that {\footnotesize $\assign{\unlock}{\actionName} \in \pre(\tilde{\action})$}.
		\item For any action set {\footnotesize $\actions_{set,{\action}}\in\tilde{\actions}$}, there is a unique action {\footnotesize  $\action_{lock}$} with {\footnotesize $\assign{\unlock}{\actionName} \in \pre(\action_{lock})$}, and a unique action {\footnotesize  $\action_{unlock}$} with {\footnotesize $\assign{\unlock}{\noOP} \in \pre(\action_{unlock})$}.			
	\end{enumerate}

	\defnotate{Case 1}. By the construction of {\footnotesize  $\tilde{\ptask}$}, in the initial state {\footnotesize  $\init$} the fact {\footnotesize  $\tuple{\unlock/\noOP}$} holds. Only actions {\footnotesize  $\action_{unlock}\in\actions_{unlock}$} contain the predicate {\footnotesize  $\tuple{\unlock/\noOP}$} in their precondition list, hence, only actions {\footnotesize  $\action_{unlock}\in\actions_{unlock}$} are applicable. Once an action {\footnotesize  $\action_{unlock}\in\actions_{unlock}$} applied an appropriate predicate {\footnotesize  $\{\tuple{\unlock/\actionName}\}$} is achieved, since the fact {\footnotesize  $\tuple{\unlock/\noOP}$} does not hold anymore no other sequence can be initiated.
	
	\defnotate{Case 2}. Similar to Case 1, a new sequence can be initiated only when the fact {\footnotesize  $\tuple{\unlock/\noOP}$} holds. Putting (1) and (2) together, then while {\footnotesize  $\tuple{\unlock/\noOP}$} holds, no other action {\footnotesize  $\tilde{{\action}}'$} can be applied such that {\footnotesize  $\tilde{{\action}}' \notin \actions_{\unlock}$}. Once an action {\footnotesize  $\action_{unlock}\in\actions_{unlock}$} applied an appropriate predicate {\footnotesize  $\{\tuple{\unlock/\actionName}\}$} is achieved, since the fact {\footnotesize  $\tuple{\unlock/\noOP}$} does not hold anymore no other sequence can be initiated.
	
	\defnotate{Case 3}. Assuming that every action {\footnotesize  $\action\in\actions$} in the original task {\footnotesize  $\ptask$} is unique, putting (1), (2) and (3) together, then as long as {\footnotesize  $\tuple{\unlock/\actionName}$} holds, no other action {\footnotesize  $\tilde{{\action}}'$} can be applied during sequence execution, such that {\footnotesize  $\tilde{{\action}}' \notin \actions_{set,{\action}}\cup\{\action_{lock}\} \mid \action \ne \action'$}. This proves \defnotate{Case 3}.
	
	\defnotate{Case 4}. At sequence termination. By the construction when {\footnotesize  $\tuple{\unlock/\actionName}$} holds, new assignment to variable {\footnotesize  $\unlock$} is applicable only with {\footnotesize  $\action_{lock}$} action, which finalizes {\footnotesize  $\tilde{\plan}_{o}$}, and achieves {\footnotesize  $\tuple{\unlock/\noOP}$}.
	
	As shown, the Lemma holds in all cases which implies that any action sequence {\footnotesize  $\tilde{\plan}_{o}$} in plan {\footnotesize  $\tilde{\plan}$} is \defnotate{atomic action sequence} such that {\footnotesize  $\nexists \tilde{{\action'}} \mid \tilde{{\action}}'  \in \tilde{\plan}_{o}, \tilde{{\action}}' \notin \actions_{set,{\action}}\cup \{\action_{unlock},\action_{lock}\}, \action \ne \action'$}.

\end{proof}

\begin{definition}
	Let {\footnotesize  $\ptask  =\langle \vars,\initstate,\valuefunc;\actions,\costfunc,\budget\rangle$} be an OSP task and {\footnotesize  $\tilde{\ptask} = \langle \tilde{\vars},\tilde{\initstate},\tilde{\valuefunc};\tilde{\actions},\tilde{\costfunc},\tilde{\budget} \rangle$}, unit-effect-compilation of {\footnotesize  $\ptask$}. The \defnotate{expanding function} {\footnotesize  $\psi : \states \mapsto \tilde{\states}$} is defined as 
	{\footnotesize  $\psi(\state) = \state \cup \{\tuple{\unlock/\noOP} \} \cup  
	\{\tuple{y/0} \mid y \in \uniqNetNegSet\}.$}
\end{definition}

\begin{lemma}
	\label{l:operators_sequence}
	Given an OSP task  {\footnotesize  $\ptask  =\langle \vars,\initstate,\valuefunc;\actions,\costfunc,\budget\rangle$}, for each action {\footnotesize  $\action \in \actions$} and each state {\footnotesize  $\state \in \states$} in which {\footnotesize  $\action$} is applicable, there is an action sequence {\footnotesize  $\tilde{\plan}_{o}$} in task {\footnotesize  $\tilde{\ptask}$}, such that:
	\begin{enumerate*}
		\item {\footnotesize  $\tilde{\plan}_{o}$} is applicable  in {\footnotesize  $\psi(\state)$},
		\item {\footnotesize  $\psi(\state)\applied{{\tilde{\plan}_{o}}}[\var] = \state\applied{\action}$},
		\item {\footnotesize  $\costfunc(\action) = \costfunc(\tilde{\plan}_{o})$}, and
		\item {\footnotesize  $\valuefunc_{\state}(\action) = \valuefunc_{{\psi(\state)}}(\tilde{\plan}_{o})$}.
	\end{enumerate*}
	
\end{lemma}

\begin{proof}
	\begin{description}
		\item [(1)] Each action sequence {\footnotesize  $\tilde{\plan}_{o}$} starts with the action {\footnotesize  $\action_{unlock}$} that allows for the execution of action sequence from action set {\footnotesize  $\actions_{set,o}$}. By the construction of {\footnotesize  $\action_{unlock}$}, applying {\footnotesize  $\action_{unlock}$} makes applicable only single-effect actions from action set {\footnotesize  $\actions_{set,o}$} while all other single-effect actions that belong to a different action sets  {\footnotesize  $\tilde{\action} \in \actions_{set,o'}$}, such that {\footnotesize  $\action \ne \action'$}, are locked. Since {\footnotesize  $\psi(\state) = \state \cup \{\tuple{\unlock/\noOP} \} \cup  
		\{\tuple{y/0} \mid y \in \uniqNetNegSet\}$} and the preconditions for {\footnotesize  $\action_{unlock}$} are defined to be {\footnotesize  $\pre(\action_{unlock}) =  \pre(\action) \cup \{\tuple{\unlock/\noOP} \}$}, the precondition for the entire sequence {\footnotesize  $\tilde{\plan}_{o}$} are applicable and {\footnotesize  $\tilde{\plan}_{o}$} is applicable in {\footnotesize  $\psi(\state)$}.

		\item [(2)] Each action sequence {\footnotesize  $\tilde{\plan}_{o}$} ends with the control action {\footnotesize  $\action_{lock}$} which verifies that  {\footnotesize  $\eff(\action)$} are achieved, and finalize action sequence {\footnotesize  $\tilde{\plan}_{o}$}.  By the construction of {\footnotesize  $\action_{lock}$, $\pre(\action_{lock})=  \eff(\action) \cup \{\tuple{\unlock/\actionName}\}$} and {\footnotesize  $\eff(\action_{lock}) =  \{\tuple{\unlock/\noOP}\} \bigcup_{\negEffectSatisfied \in \uniqNetNegSet \mid \tuple{\var/e} \in \eff(\action)} \{\tuple{\negEffectSatisfied/0}$} hence, when {\footnotesize  $\action_{lock}$} is applied, action sequence {\footnotesize  $\tilde{\plan}_{o}$} is terminated, and the following holds; {\footnotesize  $\psi(\state)\applied{{\tilde{\plan}_{o}}} = \eff(\action) \cup \{\tuple{\unlock/\noOP}\} \cup  
		\{\tuple{y/0} \mid y \in \uniqNetNegSet\}$} and, {\footnotesize  $\psi(\state)\applied{{\tilde{\plan}_{o}}}[\var] = \state\applied{\action}$}.
		
		\item [(3)] By the construction of {\footnotesize  $\tilde{\ptask}$}, all the actions {\footnotesize  $\tilde{\action} \in \tilde{\actions}$} are with zero cost except {\footnotesize  $\actions_{unlock}$} actions which carry the cost of the original action defined to be {\footnotesize  $\tilde{\costfunc}(\action_{unlock}) = \costfunc(\action)$}. Each action sequence {\footnotesize  $\tilde{\plan}_{o}$} defined to start with a unique action {\footnotesize  $\action_{unlock}$}. Hence, {\footnotesize  $\costfunc(\action) = \tilde{\costfunc}(\tilde{\plan_{o}})$}
		\item [(4)] By Definition~\ref{def:actionNetValueNoPre} of the net utility and the construction of {\footnotesize  $\action_{lock}$} it is easy to see that {\footnotesize  $\valuefunc_{\state}(\action) = \valuefunc({\action_{lock}})$}. By Lemma~\ref{l:separated}, any actions sequence {\footnotesize  $\tilde{\plan}_{o}$} in plan {\footnotesize  $\tilde{\plan}$} is {\bf atomic} in the sense that {\footnotesize  $\nexists \tilde{{\action'}} \mid \tilde{{\action}}'  \in \tilde{\plan}_{o}, \tilde{{\action}}' \notin \actions_{set,{\action}}\cup \{\action_{unlock},\action_{lock}\}, \action \ne \action'$}. Hence, by the construction of {\footnotesize  $\actions_{set,o}$}, the only utility carrying facts {\footnotesize  $\tuple{\var/e}$} that can be achieved with sequence {\footnotesize  $\tilde{\plan}_{o}$} are from the effect list of the original action {\footnotesize  $\action$}, such that {\footnotesize  $\tuple{\var/e} \in \eff(\action)$}. Since any action sequence {\footnotesize  $\tilde{\plan}_{o}$} must start with an action {\footnotesize  $\action_{unlock}$} and end with an action {\footnotesize  $\action_{lock}$} (as provided at (1) and (2) in this proof)  we can infer that {\footnotesize  $\valuefunc(\action_{lock}) = \valuefunc_{{\psi(\state)}}(\tilde{\plan}_{o})$}. Hence, we have {\footnotesize  $\valuefunc_{\state}(\action) = \valuefunc_{{\psi(\state)}}(\tilde{\plan}_{o})$}.   
	\end{description}
\end{proof}

\begin{lemma}
	\label{l:optimal_separated}
	Let {\footnotesize  $\ptask$} be an OSP task and {\footnotesize  $\tilde{\ptask}$} be the respective unit-effect-compilation of {\footnotesize  $\ptask$}. Any non-empty optimal tail gaining plan {\footnotesize  $\tilde{\plan}$} for {\footnotesize  $\tilde{\ptask}$} ends with a complete action sequence {\footnotesize  $\{\tilde{\plan_{\action}} \mid \action \in \actions\}$}. 
\end{lemma}

\begin{proof}
	By the construction of {\footnotesize  $\tilde{\ptask}$}, for each sequence {\footnotesize  $\{\tilde{\plan_{\action}} \mid \action \in \actions\}$}, it holds that:
	\begin{enumerate}
		\item Actions {\footnotesize  $\action_{unlock}$} and {\footnotesize  $\action_{lock}$} are not utility carrying facts so they are not changing utility during plan execution.
		\item The cost for the entire sequence is paid at the first action executed in {\footnotesize  $\tilde{\plan_{\action}}$}, the action {\footnotesize  $\action_{unlock}$}, while any other action {\footnotesize  $\tilde{\action} \in \tilde{\plan_{\action}}\mid \tilde{\action} \ne \action_{unlock}$} it holds that {\footnotesize  $\tilde{\costfunc}(\tilde{\action}) = 0$}.
		\item Net positive actions are applied last in the sequence after all net negative and neutral utility actions been applied due to the auxiliary control structure where each net positive action is preconditioned with achieving all control predicates {\footnotesize  $\negEffectSatisfied$, $\bigcup_{\negEffectSatisfied \in \uniqNetNegSet \mid \tuple{\var/e} \in \eff(\action)} \{\tuple{\negEffectSatisfied/1}\}.$}
		
		Each effect {\footnotesize $e \in \eff(\action)$} of the original action {\footnotesize $\action$} is verified to be achieved with net negative/neutral utility action or verified that cannot be achieved with such action, before the net positive utility actions become applicable and may achieving that effect. For every valid net negative utility achievement of an effect {\footnotesize $e$}, such that {\footnotesize $\valuefunc_{\var}(e) - \valuefunc_{\var}(p) \leq 0$} the action {\footnotesize  $\action^{verify}_{p\to e}$} is built and verifies that the effect {\footnotesize $e$} has been achieved by applying net negative/neutral utility single-effect action, this verification is by achieving the predicate {\footnotesize  ${\tuple{\negEffectSatisfied/1}}$}. For every valid positive net utility achievement of an effect {\footnotesize  $e$}, such that {\footnotesize  $\valuefunc_{\var}(e) - \valuefunc_{\var}(p) > 0$} the action {\footnotesize  $\action^{verifyNo}_{p\to e}$} is built and verifies that the effect {\footnotesize $e$} can't be achieved by applying net negative/neutral action, this verification is by achieving the predicate {\footnotesize  ${\tuple{\negEffectSatisfied/1}}$}. Only after all the effects list verified with the actions {\footnotesize  $\action^{verify}_{p\to e}$} or {\footnotesize  $\action^{verifyNo}_{p\to e}$}, the precondition {\footnotesize $\bigcup_{\negEffectSatisfied \in \uniqNetNegSet \mid \tuple{\var/e} \in \eff(\action)} \{\tuple{\negEffectSatisfied/1}\}$} for the net positive utility single-effect action holds, and net positive utility single-effect actions become applicable. 

	\end{enumerate}
	From (1), (2), and (3) it is easy to see that the auxiliary control structure within the action sequence {\footnotesize  $\tilde{\plan_{\action}}$} ensures that during action sequence execution first collected facts such that carry negative utility and then collected those carry positive utility, hence, the optimal utility will be achieved with the final action applied (before locking the sequence with action {\footnotesize  $\action_{lock}$}), at any point before the final action applied the accumulated utility is sub-optimal.
\end{proof}

\begin{theorem}
	For any optimal, tail gaining plan {\footnotesize $\plan$} for {\footnotesize $\ptask$} with cost of {\footnotesize $\costfunc(\plan) = b$} and utility {\footnotesize $\valuefunc(\plan) = \alpha$}, there is an optimal, tail gaining plan {\footnotesize $\tilde{\plan}$} for {\footnotesize $\tilde{\ptask}$}, such that {\footnotesize $\costfunc(\plan) = \costfunc(\tilde{\plan}) = b$} and {\footnotesize $\valuefunc(\plan) = \tilde{\valuefunc}(\tilde{\plan}) = \alpha$}, and vice verse. Furthermore, for any plan {\footnotesize $\tilde{\plan}$ for $\tilde{\ptask}$}, the corresponding plan {\footnotesize $\plan$} for {\footnotesize $\ptask$} can be restored.  

\end{theorem}

\begin{proof}
	\begin{description}
		\item [(1)] From the first direction, we  show that, for an optimal plan {\footnotesize $\plan$} for task {\footnotesize $\ptask$} under the budget of {\footnotesize $b$} and solution utility of {\footnotesize $\alpha$}, there exists an optimal plan {\footnotesize $\tilde{\plan}$} for {\footnotesize $\tilde{\ptask}$} with the same budget {\footnotesize $b$} and solution utility of {\footnotesize $\alpha$}. Let plan {\footnotesize $\plan = \tuple{\action_{1}, \action_{2}, \dots ,\action_{n}}$} be an optimal solution for task {\footnotesize $\ptask$} such that applicable at {\footnotesize $\init$}, with {\footnotesize $\valuefunc(\init\applied{\plan}) = \alpha$} and {\footnotesize $\costfunc(\plan) = b$}. By Lemma~\ref{l:operators_sequence} we can replace each action {\footnotesize $\action_{i} \in \plan$} with the equivalent actions sequence {\footnotesize $\tilde{\plan}_{\action_{i}}$} and get an equivalent plan {\footnotesize $\tilde{\plan} = \tuple{\tilde{\plan}_{\action_{1}}, \tilde{\plan}_{\action_{2}}, \dots ,\tilde{\plan}_{\action_{n}}}$} in {\footnotesize $\tilde{\ptask}$} , where,
		\begin{itemize}
			\item for any action {\footnotesize $\action_i \in \plan$} in the original plan and the equivalent action sequence {\footnotesize $\tilde{\plan}_{\action_{i}} \in \tilde{\plan}$} in {\footnotesize $\tilde{\plan}$} holds that {\footnotesize $\costfunc(\action_{i}) = \costfunc(\tilde{\plan}_{o_i})$}, hence by accumulating all actions costs in {\footnotesize $\plan$} and accordingly action sequence cost in {\footnotesize $\tilde{\plan}$} we get {\footnotesize $\costfunc(\plan) = \costfunc(\tilde{\plan}) = b$}.
			\item {\footnotesize $\tilde{\plan}$} is applicable at {\footnotesize $\psi(\init)$} and ends with state {\footnotesize $\tilde{\state} =\psi(\init\applied{\plan})$} such that {\footnotesize $\valuefunc(\tilde{\state}) = \valuefunc(\init\applied{\plan}) = \alpha$}. 
		\end{itemize}
		Now we will show that there is no other plan {\footnotesize $\tilde{\plan}$} for {\footnotesize $\tilde{\ptask}$} that achieves utility beyond {\footnotesize $\alpha$}. By Lemma~\ref{l:separated} and Lemma~\ref{l:optimal_separated}, optimal plan for {\footnotesize $\plan$} will be reached only after completion of full sequences (or not at all) from {\footnotesize $\{\tilde{\plan_{\action}} \mid \action \in \actions\}$}, i.e. optimal plan couldn't be reached in the middle of any sequence {\footnotesize $\tilde{\plan_{\action}} \mid \action \in \actions$}. Furthermore, by Lemma~\ref{l:operators_sequence} for each sequence {\footnotesize $\tilde{\plan}_{\action_{i}}$} there is a mapping to {\footnotesize $\action \in \actions$}. 
		
		Now let assume to the contrary that there is a plan {\footnotesize $\tilde{\plan}' = \tuple{\tilde{\plan}'_{\action_{1}}, \tilde{\plan}'_{\action_{2}}, \dots ,\tilde{\plan}'_{\action_{k}}}$} for {\footnotesize $\tilde{\ptask}$} such that achieves utility {\footnotesize $\beta$} and {\footnotesize $\beta > \alpha$}. By Lemma~\ref{l:separated}, Lemma~\ref{l:operators_sequence} and Lemma~\ref{l:optimal_separated}, from the optimal plan there is a mapping to a plan in the original task {\footnotesize $\ptask$} that achieves the same utility, {\footnotesize $\beta$} under the same budget. The mapping for {\footnotesize $\tilde{\plan}'$} is  {\footnotesize $\plan' = \tuple{\action_{1}', \action_{2}', \dots ,\action_{k}'}$} with {\footnotesize $\valuefunc(\init\applied{\plan'}) = \beta$}. This, however contradicting the optimality of {\footnotesize $\plan$} for {\footnotesize $\ptask$}.
		
		\item [(2)] The proof of the other direction is rather similar. We show that, for an optimal plan {\footnotesize $\tilde{\plan}$} for the compiled task {\footnotesize $\tilde{\ptask}$} under the budget of {\footnotesize $b$} and solution utility of {\footnotesize $\alpha$}, there exists an optimal plan {\footnotesize $\plan$} for {\footnotesize $\ptask$} with the same budget {\footnotesize $b$} and solution utility of {\footnotesize $\alpha$}.  Let {\footnotesize $\tilde{\plan}$} be an optimal plan for the compiled task {\footnotesize $\tilde{\ptask}$}, then by Lemma~\ref{l:separated} and Lemma~\ref{l:optimal_separated} this optimal plan is made of full and separated sequences from {\footnotesize $\{\tilde{\plan_{\action}} \mid \action \in \actions\}$}, i.e. it is of the from, {\footnotesize $\tilde{\plan} = \tuple{\tilde{\plan}_{\action_{1}}, \tilde{\plan}_{\action_{2}}, \dots ,\tilde{\plan}_{\action_{n}}}$}. By Lemma~\ref{l:operators_sequence} (and step (1)) we can map {\footnotesize $\tilde{\plan}$} to an equivalent solution {\footnotesize $\plan = \tuple{\action_{1}, \action_{2}, \dots ,\action_{n}}$} in the original task {\footnotesize $\ptask$} such that achieves the same utility {\footnotesize $\alpha$} under the same budget {\footnotesize $b$}.
		
		Now let assume to the contrary that there is a plan {\footnotesize $\plan' = \tuple{\action'_{1}, \action'_{2}, \dots ,\action_{k}'}$} for the original task {\footnotesize $\ptask$} such that achieves utility {\footnotesize $\beta$} and {\footnotesize $\beta > \alpha$}. By Lemma~\ref{l:separated},~\ref{l:operators_sequence} and~\ref{l:optimal_separated}, there is a mapping from the optimal plan {\footnotesize $\plan'$} in the original task to a plan {\footnotesize $\tilde{\plan}'$} in the compiled task {\footnotesize $\tilde{\ptask}$} that achieves the same utility, {\footnotesize $\beta$} under the same budget. The mapping for {\footnotesize $\plan'$} is  {\footnotesize $\tilde{\plan}' = \tuple{\tilde{\plan}'_{\action_{1}}, \tilde{\plan}'_{\action_{2}}, \dots ,\tilde{\plan}'_{\action_{k}}}$} with {\footnotesize $\valuefunc(\init\applied{\tilde{\plan}'}) = \beta$}. This, however contradicting the optimality of {\footnotesize $\tilde{\plan}$} for {\footnotesize $\tilde{\ptask}$}.
	\end{description}
\end{proof}

\subsection{Recognition of Real-Time Tactical Decisions}
\label{eval:unitsplit}
To address real-world scenarios where the polarity of an action can be detected only with interaction with the environment
we have incorporated unit effect action split into the pre-processing phase of the selective action split. The unit-effect action split can be activated just for a subset of action that fit some structural criteria. 

In order to obtain the optimal split decision criteria we should examine the structure of the affected domains and actions to obtain the structure related criteria to activate unit-effect split. Since there is a trade off between the increase in task size obtain with each approach to split, we compare the expected increase in the size of the task obtained with each approach and set the critical value for the decision with relation to the increase in size task caused by each split method. This comparison is easily obtained during the pre-processing phase in the selective action split procedure, or in the translation of PDDL representation to SAS+ representation.

\begin{figure}[t]
		{ \centering
			\scriptsize
			\begin{tabbing}
				\underline{{\bf Selective-Action-Split}} ($\ptask = \langle \vars,\initstate,\valuefunc;\actions,\costfunc,\budget\rangle$)\\
				\;\;\= mutex-Inference($\ptask$)\\
				\;\;\= $\actions_{new} = \emptyset$\\
				\> {\bf for each} \= $\action \in \actions$;\\
				\>\;\;\;\=\mbox{\em \textcolor{blue}{\footnotesize // explicit net utility (for all known preconditions)}}\\
				\>\;\;\;\= $\text{ENU} := \sum_{\var\in\variables{\pre(\action)}}[{\valuefunc(\eff(\action)[\var])}-{\valuefunc(\pre(\action)[\var])}]$\\
				\>\;\;\;\= $\text{MSS := {\bf min-split-set}}(\action)$\\
				\>\;\;\;\= $\actions_{new} := \actions_{new} \cup \text{{\bf marked-or-splitted}(MSS,ENU)}$\\
				\\
				\underline{{\bf min-split-set}}($\action$)\\
				\> {\bf return} $\{\var\mid\var\in\variables{\eff(\action)\setminus\pre(\action)}, \text{{\bf refined-domain}}(\var) \neq \emptyset\}$\\
				\\        
				\underline{{\bf refined-domain}}($\var,\action$)\\
				\>\=\mbox{\em \textcolor{blue}{\footnotesize // remove variables without utility variance}}\\
				\>\=\mbox{\em \textcolor{blue}{\footnotesize // remove facts that are mutex with known preconditions}}\\
				\> {\bf return} $\{\val\mid\val\in\domain(\var),\valuefunc(\tuple{\var/\val}) \neq \valuefunc(\eff(\action)[\var]),$\\
				\>\;\;\;\;\;\;\;\;\;\;\;\;\;\;\;\= $\tuple{\var/\val}\notin\text{mutex-group}(\variables{\pre(\action)})\}$\\
				\\
				\underline{{\bf marked-or-split}}(MSS,ENU)\\
				
				\>\=\mbox{\em \textcolor{blue}{\footnotesize // max/min floating net utility (for all unknown preconditions)}}\\    
				\>\= $\text{MAXFNU} := \sum_{\var\in\text{MSS}}[{\valuefunc(\eff(\action)[\var])}-\max\limits_{\forall d \in \domain(\var)}\valuefunc(\tuple{\var/\val})]$\\        
				\>\= $\text{MINFNU} := \sum_{\var\in\text{MSS}}[{\valuefunc(\eff(\action)[\var])}-\min\limits_{\forall d \in \domain(\var)}\valuefunc(\tuple{\var/\val})]$\\
				\>\= {\bf if} MAXFNU + ENU $\leq$ 0;\\
				\>\;\;\;\;\;\;\= {\bf return} $\action$-neg \mbox{\em \textcolor{blue}{\footnotesize // mark without split}}\\
				\>\= {\bf if} MINFNU + ENU $>$ 0;\\
				\>\;\;\;\;\;\;\= {\bf return} $\action$-pos \mbox{\em \textcolor{blue}{\footnotesize // mark without split}}\\
				\> {\bf return} {\bf $\action$-classifier}(MSS, ENU)\\
				\\
				\underline{{\bf $\action$-classifier}}(MSS, ENU)\\
				any binary multivariate classification algorithm into two groups:\\
				\> 1) net positive actions, and\\
				\> 2) not net positive actions \\
			\end{tabbing}
		}
		\caption{\label{fig:Selective} Selective-Action-Split translation procedure}
\end{figure}

The pre-processing procedure to refine value-ambiguity, incorporating the unit effect split with the selective action split, is as follows;

\begin{enumerate}
	\item Perform a {\bf mutex inference} to complete precondition list as much as possible (with standard Fast Downward mechanism) to reduce the computational effort.
	\item Calculate the {\bf explicit net utility (ENU)} 
	\begin{description}
		\item [a.] Actions that carry negative or zero net utility for all their instances in the normal form can remain in their compact encoding.
		\item [b.] Actions that are {\em pure positive}, such that carry positive net utility for all their instances in the normal form can be marked as {\em goal actions} and remain in their compact encoding as well
	\end{description}
	\item Reduce remained variables with {\bf min-split-set}
	\begin{description}
		\item [a.] Variables with no utility variance (regardless the numerical utility value) are removed since [b.] they have no net utility contribution 
		\item [c.] Perform mutex inference of unknown preconditions with reveled precondition facts; remove if found.
	\end{description}
	\item If redundant preconditions left (net utility signum have not determined);
	\begin{description}
		\item [a.] Calculate the maximal and the minimal floating net utility ({\bf MAXFNU, MINFNU})
		\item [b.] If utility signum remains positive borders; mark positive action,
		\item [c.] Else, classify with {\bf binary-multivariate-classifier}. To split such actions, we first calculate the total net utility of variables with a known precondition. Suppose we got a value of {\footnotesize $x$}. Next step, just for variables with an unrestricted precondition, we divide to 2 groups, the positive group with higher total sum than {\footnotesize $-x$} and appositive group with lower (or equal) total sum than {\footnotesize $-x$}. This split can be done with any method of binary multivariate classification.
	\end{description}
\end{enumerate}

In other words, we may recognize positive net actions off-line, and perform a normal form encoding (or transition normal form encoding) only for actions that are ambiguous with regard to their net utility signum. Considering the structure of the precondition list of an action we define three versions of the selective action split for online analysis of net utility polarity as follows.
\begin{description}
	\item [{\footnotesize $S_{base}$} -] selective action split without unit effect split.
	\item [{\footnotesize $S_{blind}$} -] unit effect action split is defined to be activated for each action which is determined as candidate to split into net positive and not net positive instances.
	\item [{\footnotesize $S_{preTotal}$} -] a unit-effect action split defined to be activated if the expected number of split actions is bigger than the total number of the preconditions of actions that expected to split with the selective action split.     
\end{description}

\begin{table*}[hbt!] 
	\centering
	\setlength{\tabcolsep}{.18em}
	{\scriptsize
		\resizebox{2.1\columnwidth}{!}{%
			\begin{tabular}{l||c|c|c|c|c|c|c|c|c|c|c|c|c|c|c|c|c|c|c|c|c|c|c|c|c|c|c|c|c|c|c|c|c|c|c|c|c|c|c|c|c|c|}	
            & \multicolumn{9}{c|}{\sl $25\%$}&\multicolumn{9}{c|}{\sl  $50\%$}& \multicolumn{9}{c|}{\sl $75\%$} & \multicolumn{9}{c|}{\sl $100\%$} \\
            \cline{2-37}
				& \multicolumn{3}{c|}{\sl $S_{base}$}&\multicolumn{3}{c|}{\sl  $S_{blind}$}& \multicolumn{3}{c|}{\sl $S_{preTotal}$} & \multicolumn{3}{c|}{\sl $S_{base}$}&\multicolumn{3}{c|}{\sl  $S_{blind}$}& \multicolumn{3}{c|}{\sl $S_{preTotal}$} & \multicolumn{3}{c|}{\sl $S_{base}$}&\multicolumn{3}{c|}{\sl  $S_{blind}$}& \multicolumn{3}{c|}{\sl $S_{preTotal}$} & \multicolumn{3}{c|}{\sl $S_{base}$}&\multicolumn{3}{c|}{\sl  $S_{blind}$}& \multicolumn{3}{c|}{\sl $S_{preTotal}$} \\
           \cline{2-37}
		&	 {\sl Exp}	&	{\sl Time}	&	{\sl Sol}	&	 {\sl Exp}	&	{\sl Time}	&	{\sl Sol}	&	 {\sl Exp}	&	{\sl Time}	&	{\sl Sol}	&	 {\sl Exp}	&	{\sl Time}	&	{\sl Sol}	&	 {\sl Exp}	&	{\sl Time}	&	{\sl Sol}	&	 {\sl Exp}	&	{\sl Time}	&	{\sl Sol}	&	 {\sl Exp}	&	{\sl Time}	&	{\sl Sol}	&	 {\sl Exp}	&	{\sl Time}	&	{\sl Sol}	&	 {\sl Exp}	&	{\sl Time}	&	{\sl Sol}	&	 {\sl Exp}	&	{\sl Time}	&	{\sl Sol}	&	 {\sl Exp}	&	{\sl Time}	&	{\sl Sol}	&	 {\sl Exp}	&	{\sl Time}	&	{\sl Sol}\\	
 \hline  																																																																									
airport(30)	&	{\bf 503282}	&	261.19	&	26	&	503303	&	256.67 &	26	&	503303	&	{\bf 250.72}	&	26	&	{\bf 1519491}	&	918.91	&	21	&	1519547	&	{\bf 916.30}	&	21	&	1519547	&	918.29	&	21	&	1596024	&	652.02	&	19	&	{\bf 1595959}	&	710.06	&	19	&	1596024	&	{\bf 630.70}	&	19	&	{\bf 2202637}	&	920.12	&	19	&	2202665	&	885.72	&	19	&	2202665	&	{\bf 855.72}	&	19	\\
blocks(26)	&	5454	&	{\bf 41.34}	&	26	&	5454	&	42.25	&	26	&	5454	&	41.77	&	26	&	10436158	&	{\bf 454.81}	&	26	&	10436158	&	466.69	&	26	&	10436158	&	461.92	&	26	&	4287047	&	323.20	&	18	&	4287047	&	326.27	&	18	&	4287047	&	{\bf 306.53}	&	18	&	14029253	&	918.95	&	17	&	14029253	&	 853.38	&	17	&	14029253	&	{\bf 830.19}	&	17	\\
depot(7)	&	2338	&	{\bf 20.36}	&	7	&	2338	&	20.93	&	7	&	2338	&	21.12	&	7	&	441952	&	75.57	&	7	&	441952	&	74.05	&	7	&	441952	&	{\bf 72.11}	&	7	&	8640604	&	878.34	&	6	&	8640604	&	919.94	&	6	&	8640604	&	{\bf 832.49}	&	6	&	5431953	&	{\bf 512.92}	&	4	&	5431953	&	575.15	&	4	&	5431953	&	564.89	&	4	\\
driverlog(14)	&	27379	&	33.65	&	14	&	27379	&	30.23	&	14	&	27379	&	{\bf 31.03}	&	14	&	3347713	&	245.75	&	12	&	3347713	&	205.18	&	{\bf 13}	&	3347713	&	{\bf 198.54}	&	{\bf 13}	&	9813940	& 527.10	&	10	&	9813940	&	565.08	&	10	&	9813940	&	{\bf 524.72}	&	10	&	6197329	&	517.74	&	7	&	6197329	&	527.74	&	7	&	6197329	&	500.86	&	7	\\
freecell(43)	&	140125	&	228.36	&	43	&	140125	&	218.23 &	43	&	140125	&	{\bf 217.06}	&	43	&	{\bf 29637127}	&	2368.01	&	31	&	29637135	& 2347.79	&	31	&	29637135	&	{\bf 2322.25}	&	31	&	4865955	&	{\bf 573.96}	&	15	&	{\bf 4865941}	&	608.94	&	15	&	4865955	&	{\bf 564.54}	&	15	&	{\bf 9365088}	&	{\bf 947.30}	&	14	&	9365112	&	992.5	&	14	&	9365112	&	970.63	&	14	\\
grid(2)	&	0	&	{\bf 6.59}	&	2	&	0	&	10.78	&	2	&	0	&	10.72	&	2	&	2020	&	{\bf 6.60}	&	2	&	2020	&	12.5	&	2	&	2020	&	9.81	&	2	&	64270	&	24.76	&	2	&	64270	&	{\bf 15.25}	&	2	&	64270	&	23.21	&	2	&	1267	&	{\bf 1.77}	&	1	&	1267	&	2.32	&	1	&	1267	&	2.22	&	1	\\
gripper(7)	&	17450	&	89.25	&	5	&	{\bf 17112}	&	{\bf 87.82}	&	5	&	17450	&	87.35	&	5	&	64520	&	{\bf 17.07}	&	4	&	{\bf 64442}	&	18.56	&	4	&	64520	&	17.65	&	4	&	263519	&	32.53	&	4	&	{\bf 263282}	&	{\bf 27.86}	&	4	&	263282	&	32.44	&	4	&	{\bf 416250}	&	{\bf 40.08}	&	4	&	417436	&	49.01	&	4	&	{\bf 416250}	&	44.96	&	4	\\
logistics00(20)	&	1119070	&	65.66	&	20	&	1119070	&	62.04	&	20	&	1119070	&	{\bf 60.54}	&	20	&	18410615	&	930.74	&	16	&	18410615	&	885.44	&	16	&	18410615	&	{\bf 861.81}	&	16	&	16682657	&	792.75	&	12	&	16682657	&	822.9	&	12	&	16682657	&	{\bf 765.90}	&	12	&	688210	&	{\bf 61.39}	&	10	&	688210	&	66.11	&	10	&	688210	&	68.4	&	10	\\
logistics98(6)	&	57297	&	{\bf 14.04}	&	6	&	57297	&	15.28	&	6	&	57297	&	14.27	&	6	&	1985934	&	{\bf 114.57}	&	4	&	1985934	&	177.03	&	4	&	1985934	&	170.32	&	4	&	52155	& 9.65	&	2	&	52155	&	9.9	&	2	&	52155	&	{\bf 9.09}	&	2	&	97889	&	13.51	&	2	&	97889	&	12.66	&	2	&	97889	&	{\bf 11.46}	&	2	\\
miconic(69)	&	2108009	&	2446.34	&	65	&	{\bf 2080805}	&	2436.04	&	65	&	{\bf 2080805}	&	{\bf 2447.19}	&	65	&	31263975	&	1688.79	&	55	&	{\bf 30533261}	&	 1613.26	&	55	&	{\bf 30533261}	&	{\bf 1595.48}	&	55	&	{\bf 34526651}	&	{\bf 2244.35}	&	{\bf 48}	&	35286750	&	2420.7	&	46	&	 60318226	&	3374.71	&	{\bf 48}	&	{\bf 7188067}	&	{\bf 781.66}	&	40	&	8585558	&	890.83	&	{\bf 41}	&	8585558	&	859	&	{\bf 41}	\\
movie(30)	&	300	&	48.40	&	30	&	{\bf 270}	&	50.84	&	30	&	300	&	{\bf 43.97}	&	30	&	3150	& 96.50	&	30	&	{\bf 3120}	&	99.13	&	30	&	3150	&	{\bf 95.4}	&	30	&	{\bf 9690}	&	151.5	&	30	&	9780	&	{\bf 140.66}	&	30	&	9780	&	142.26	&	30	&	13650	&	{\bf 171.01}	&	30	&	13650	&	178.66	&	30	&	13650	&	169.19	&	30	\\
mprime(19)	&	0	&	{\bf 73.11}	&	19	&	0	&	76.85	&	19	&	0	&	77.7	&	19	&	6217	&	84.94	&	19	&	6217	&	88.02	&	19	&	6217	&	{\bf 83.39}	&	19	&	114171	&	115.24	&	19	&	114171	&	116.07	&	19	&	114171	&	{\bf 114.61}	&	19	&	778740	&	{\bf 194.06}	&	18	&	778740	&	214.4	&	18	&	778740	&	206.78	&	18	\\
mystery(15)	&	0	&	{\bf 88.42}	&	15	&	0	&	97.99	&	15	&	0	&	91.17	&	15	&	343	&	104.25	&	15	&	343	&	103.68	&	15	&	343	&	{\bf 99.45}	&	15	&	16440	&	116.53	&	15	&	16440	&	{\bf 109.13}	&	15	&	16440	&	113.89	&	15	&	1835459	&	284.42	&	15	&	1835459	&	279.86	&	15	&	1835459	&	271.8	&	15	\\
openstacks(7)	&	15194	&	{\bf 7.19}	&	7	&	15194	&	7.31	&	7	&	15194	&	7.23	&	7	&	979207	&	{\bf 93.76}	&	7	&	979207	&	103.59	&	7	&	979207	&	99.94	&	7	&	4522208	&	404.41	&	7	&	4522208	&	408	&	7	&	4522208	&	{\bf 394.76}	&	7	&	80113	&	{\bf 14.29}	&	5	&	80113	&	17.32	&	5	&	80113	&	17.51	&	5	\\
pipesw-nt(19)	&	{\bf 4310}	&	90.11	&	19	&	6172	&	{\bf 92.02}	&	19	&	6172	&	96.86	&	19	&	{\bf 514066}	&	{\bf 239.85}	&	19	&	849912	&	298.77	&	19	&	849912	&	282.4	&	19	&	2194830	&	789.78	&	17	&	{\bf 1323608}	&	{\bf 538.54}	&	{\bf 18}	&	2194830	&	778.85	&	17	&	{\bf 2685297}	&	664.56	&	12	&	2733865	&	{\bf 660.13}	&	12	&	2733865	&	647.48	&	12	\\
pipesw-t(13)	&	3715	&	115.61	&	13	&	{\bf 3699}	&	{\bf 102.18}	&	13	&	{\bf 3699}	&	102.50	&	13	&	{\bf 563299}	&	230.2	&	13	&	602623	&	197.63	&	13	&	602623	&	{\bf 193.60}	&	13	&	6718726	& 827.96 &	12	&	{\bf 6638989}	&	902.34	&	12	&	6718726	&	{\bf 821.47}	&	12	&	{\bf 5412909}	&	547.08	&	{\bf 10}	&	5494570	&	535.39	&	9	&	5494570	&	{\bf 513.72}	&	9	\\
psr-small(49)	&	{\bf 1235}	&	304.26	&	49	&	1320	&	361.18	&	49	&	{\bf 1235}	&	{\bf 191.16}	&	49	&	{\bf 64007}	&	{\bf 316.57}	&	49	&	69039	&	362.46	&	49	&	{\bf 64007} 	&	{\bf 198.13}	&	49	&	7010026	&	810.48	&	49	&	{\bf 6880198}	&	 771.20	&	49	&	{\bf 6880198}	&	{\bf 611.67}	&	49	&	{\bf 3772563}	&	549.34	&	48	&	4049372	&	587.01	&	48	&	{\bf 3772563}	&	{\bf 447.33}	&	48	\\
rovers(8)	&	{\bf 83542}	&	{\bf 24.93}	&	8	&	108090	&	25.71	&	8	&	{\bf 83542}	&	26.71	&	8	&	{\bf 1177111}	&	84.41	&	7	&	1825793	&	106.05	&	7	&	{\bf 1177111}	&	{\bf 84.28}	&	7	&	5252516	&	397.45	&	6	&	{\bf 4416861}	&	{\bf 319.85}	&	6	&		{\bf 4416861}	&	373.78	&	6	&	{\bf 2779}	&	3.92	&	4	&	5715	&	4.55	&	4	&	{\bf 2779}	&	{\bf 3.65}	&	4	\\
satellite(7)	&	{\bf 3884}	&	{\bf 12.19}	&	7	&	111932	&	19.78	&	7	&	{\bf 3884}	&	12.65	&	7	&	1032509	&	126.21	&	6	&	{\bf 1017953}	&	{\bf 114.27}	&	6	&	1032509	&	123.1	&	6	&	180146	&	19.3	&	4	&	{\bf 160173}	&	{\bf 17.95}	&	4	&	{\bf 160173}	&	20.37	&	{\bf 5}	&	{\bf 520289}	&	{\bf 55.34}	&	4	&	635635	&	64.37	&	4	&		{\bf 520289} 	&	60.18	&	4	\\
storage(15)	&	427	&	{\bf 27.73}	&	15	&	{\bf 269}	&	28.52	&	15	&	{\bf 269}	&	27.9	&	15	&	206429	&	64.78	&	15	&	{\bf 84665}	&	57.04	&	15	&	{\bf 84665}	&	{\bf 53.43}	&	15	&	{\bf 2088654}	&	221.27	&	15	&	5244618	&	494.19	&	15	&	{\bf 2088654}	&	{\bf 214.41}	&	15	&	{\bf 3202242}	&	{\bf 326.25}	&	14	&	3507853	&	356.41	&	14	&	3507853	&	342.9	&	14	\\
tpp(6)	&	311	&	{\bf 3.68}	&	6	&	311	&	3.95	&	6	&	311	&	3.95	&	6	&	101221	&	{\bf 11.80}	&	6	&	101221	&	14.83	&	6	&	101221	&	14.51	&	6	&	2095337	&	206.86	&	6	&	2095337	&	{\bf 160.80}	&	6	&	2095337	&	209.94	&	6	&	36648	&	{\bf 6.57}	&	5	&	36648	&	7.9	&	5	&	36648	&	8.09	&	5	\\
trucks-strips(8)	&	38473	&	27.38	&	8	&	{\bf 29851}	&	{\bf 27.17}	&	8	&	{\bf 29851}	&	27.66	&	8	&	7065723	&	600.02	&	7	&	{\bf 4452166}	&	{\bf 341.06}	&	7	&	{\bf 4452166}	&	346.03	&	7	&	{\bf 1620220}	&	180.98	&	{\bf 5}	&	2969121	&	242.4	&	4	&	{\bf 1620220}	&	{\bf 168.03}	&	{\bf 5}	&	5138304	&	467.46	&	4	&	{\bf 4058111}	&	421.50	&	4	&	{\bf 4058111}	&	{\bf 415.34}	&	4	\\
zenotravel(10)	&	6960	&	16.54	&	10	&	6960	&	17.62	&	10	&	6960	&	{\bf 15.43}	&	10	&	5013548	&	415.17	&	10	&	5013548	&	352.37	&	10	&	5013548	&	{\bf 315.64}	&	10	&	280575	&	38.85	&	8	&	280575	&	{\bf 33.98}	&	8	&	280575	&	35.16	&	8	&	2277262	&	{\bf 255.95}	&	8	&	2277262	&	273.86	&	8	&	2277262	&	250.81	&	8	\\
\hline																																																																									
total(430)	&	4138755	&	4046.33	&	420	&	4236951	&	4091.39	&	420	&	{\bf 4104638}	&	{\bf 3906.66}	&	420	&	113836335	&	9289.28	&	381	&	111384584	& 8955.7	&	{\bf 382}	&	{\bf 110745534}	&	{\bf 8617.48}	&	{\bf 382}	&	{\bf 112896361}	&	{\bf 10339.27}	&	329	&	116224684	&	10682.01	&	327	&	137702333	&	11063.53	&	{\bf 330}	&	{\bf 71374198}	&	8255.69	&	295	&	72523665	&	8456.78	&	295	&	72127388	&	{\bf 8063.11}	&	295	\\
				\hline				
			\end{tabular}
		}
	}
	\caption{\label{results:unitsplit} Performance of the selective action split baseline {\footnotesize $S_{base}$} version vs. {\footnotesize $S_{blind}$} version, with 25\% budgets relative to {\footnotesize $c^\ast$}.}
\end{table*}

\section{Empirical Evaluation}
\label{chap:ecaluations} 
As OSP still lacks a standard suite of benchmarks for comparative evaluation, we have cast in this role the STRIPS classical planning tasks from the International Planning Competitions (IPC) 1998-2006. This ``translation'' to OSP was done by associating a separate value (0,1 and 2) with each fact in the corresponding classical IPC task. Note, the selective action split manages to determine the net utility of actions perfectly, without split, in most of the domains and tasks. Having said that, we expect no change in the performance in domains that action split appeared to be unnecessary. As expected the performance in most of the domains remained similar. To have a closer look on task on which  the unit-effect split is targeted we have to refine the data in order to capture the insights behind the data. We now describe the results for each of these action split methods.

Tables~\ref{results:unitsplit} compare the expanded nodes, total time and total number of problems solved by the {\em improving} approach applying the baseline {\footnotesize $S_{base}$} split version versus the {\footnotesize $S_{blind}$} split version and the the {\footnotesize $S_{preTotal}$} unit-effect split version. As these Tables show, for task with budget of 50\% and 75\%, {\footnotesize $S_{blind}$} improved the performance of the selective action split in terms of total expanded nodes and time, and managed to solve 3 tasks more in total. In tasks with budget of 25\% and 100\% the {\footnotesize $S_{base}$} performed better than the {\footnotesize $S_{blind}$} in terms of total expanded nodes and time. 

\begin{table}[t]
	\centering
	\setlength{\tabcolsep}{.18em}
	{\scriptsize
		\resizebox{0.7\columnwidth}{!}{%
			\begin{tabular}{l||c|c|c|c||c|c|c|c|}	
				&\multicolumn{4}{c|}{\sl Rel. Change $S_{base} \rightarrow S_{blind}$  }&\multicolumn{4}{c|}{\sl  Rel. Change $S_{base} \rightarrow S_{preTotal}$  }\\
                  \cline{2-9}
				& {\sl 25\%}& {\sl 50\%}& {\sl 75\%} &  {\sl 100\%}& {\sl 25\%}& {\sl 50\%}& {\sl 75\%} &  {\sl 100\%}\\
				\hline									
miconic	&	0	&	0	&	0	&	{\bf+19\%}	&	0	&	0	&	0	&	{\bf+18\%}	\\
pipesw-nt	&	{\bf+43\%}	&	0	&	{\bf+65\%}	&	{\bf+66\%}	&	{\bf+43\%}	&	{\bf+65\%}	&	{\bf+66\%}	&	0	\\
\hline  		
\hline  	
rovers	&	{\bf+29\%}	&	{\bf+55\%}	&	{\bf+19\%}	&	{\bf+106\%}	&		{\color{blue} \bf 0}	&	{\color{blue} \bf 0}	&{\color{blue} \bf 0}	&	{\color{blue} \bf 0}	\\
satellite	&	{\bf+2782\%}	&	{\bf+12\%}	&	0	&	{\bf+22\%}	&	{\color{blue} \bf 0}	&	{\color{blue} \bf 0}	&	{\color{blue} \bf 0}	&	{\color{blue} \bf 0}	\\
\hline  	
\hline  	
storage	&	{\color{red} \bf-37\%}	&	{\color{red} \bf-59\%}	&	{\color{red} \bf-60\%}	&	0	&	{\color{red} \bf-37\%}	&	{\color{red} \bf-59\%}	&	{\color{red} \bf-60\%}	&	0	\\
trucks	&	{\color{red} \bf-22\%}	&	{\color{red} \bf-37\%}	&	{\color{red} \bf-45\%}	&	{\color{red} \bf-21\%}	&	{\color{red} \bf-22\%}	&	{\color{red} \bf-37\%}	&	{\color{red} \bf-45\%}	&	{\color{red} \bf-21\%}	\\
				\hline  							
			\end{tabular}
		}
	}
	\caption{\label{results:table:relative:change} Relative change in the number of expended nodes with activation of {\footnotesize $S_{blind}$} version of the selective action split relative to the baseline {\footnotesize $S_{base}$} version across all budgets with a 10\% threshold.}
\end{table}

Table~\ref{results:table:relative:change} provides a better look on the machinery of unit-effect split and actions structure. This table shows the relative change in terms of expanded nodes due to the activation of {\footnotesize $S_{blind}$}. We set a threshold of 10\% to filter domains that had significant relative change. Iterative evaluation of significant level of relative change will allow us to learn properties of optimal split and a properties based, critical decision criteria for the choice of a unit-effect action split method. 
The results in red bold font represent a relative decrease in the number of expanded nodes due to activation of the unit-effect split and the results in bold black font represent a relative increase in the number of expanded node. With this representation it is easy to see that the impact of the unit-effect split depends on domain specifics. More specifically, the decision to apply unit-effect compilations depends on action structure and proprieties of actions. As Table~\ref{results:table:relative:change} shows, activation of {\footnotesize $S_{blind}$} allowed for significant improvement in {\em storage} and {\em trucks} domains, and retreat in the performance in {\em miconic}, {\em pipesworld-notankage}, {\em rovers} and {\em satellite} domains. 
We compared also {\footnotesize $S_{preTotal}$} unit-effect split version of selective action split versus the baseline {\footnotesize $S_{base}$} selective action split versus the {\footnotesize $S_{preTotal}$}.  As these Tables show, for task with budget of 25\%, 50\% and 75\%, {\footnotesize $S_{preTotal}$} improved the performance of the selective action split in terms of total expanded nodes and time and managed to solve 2 tasks more in total. In tasks with budget of 100\% the {\footnotesize $S_{base}$} performed better {\footnotesize $S_{preTotal}$} in terms of total expanded nodes and time.


Table~\ref{results:unitsplit} shows the relative change in terms of expanded nodes due to the activation of {\footnotesize $S_{preTotal}$}. We set a threshold of 10\% to filter domains that had significant relative change. 
The results in red bold font represent a relative decrease in the number of expanded nodes due to activation of the unit-effect split and the results in bold black font represent a relative increase in the number of expanded node. 
As Table~\ref{results:table:relative:change} shows, with activation of {\footnotesize $S_{preTotal}$} unit-effect split, the performance improvement in {\em storage} and {\em trucks} domains, remained while, additionally the retreat in performance of {\em rovers} and {\em satellite} domains that observed in previous experiment, with {\footnotesize $S_{blind}$} selective action split was fixed.

\section{Conclusion and Future Work}

While the initial reason to split actions is to define the net utility of an action, there is a secondary effect that we must take in account when we split actions. Recall, we split actions when they have few instances with different net utility signs. Suppose an action {\footnotesize $\action$} appears in a landmark {\footnotesize $\disland$}. Then standard landmark backchaining techniques 
will allow us to discover more landmarks back from {\footnotesize $\disland$}. However, if we know that {\footnotesize $\action$} must yield positive net utility in {\footnotesize $\disland$}, by splitting we gain more information about the state which must hold when {\footnotesize $\action$} is applied, therefore yielding more informative landmarks. For future work, we will extend the experiments to more complex domains and utility settings and analyze the structure properties of split actions with an objective to obtain optimal criteria for unit-effect action split. Here we evaluated two versions of unit-effect split in which the decision to activated the unit split is per-task bulk decision. In future work we will implement a machinery of per-action unit split decision which will allow for a better insights and exploitation of action properties in the context of choosing optimal split method.

\bibliographystyle{aaai}
\bibliography{muller}
\end{document}